\documentclass{article}

\usepackage{microtype}
\usepackage{graphicx}
\usepackage{subcaption}
\usepackage{booktabs} 
\usepackage{multicol}
\usepackage{multirow}
\usepackage{natbib}
\usepackage{hyperref}

\usepackage[accepted]{icml2025}

\usepackage{enumitem}
\usepackage{amsmath}
\usepackage{amssymb}
\usepackage{mathtools}
\usepackage{amsthm}

\usepackage[capitalize,noabbrev]{cleveref}

\theoremstyle{plain}
\newtheorem{theorem}{Theorem}[section]
\newtheorem{proposition}[theorem]{Proposition}
\newtheorem{lemma}[theorem]{Lemma}

\theoremstyle{plain}
\newtheorem{definition}{\protect Definition}

\theoremstyle{remark}

\usepackage[textsize=tiny]{todonotes}
\usepackage{tikz}
\usetikzlibrary{bayesnet}
\usetikzlibrary{arrows}

\icmltitlerunning{Deferring Concept Bottleneck Models}

\usepackage{amsmath,amsfonts,bm}
\usepackage{dsfont, pifont}
\usepackage{xcolor}

\def\eqref#1{equation~\ref{#1}}
\def\Eqref#1{Equation~\ref{#1}}

\def\1{\bm{1}}

\def\rvx{{\vec{x}}}

\DeclareMathAlphabet{\mathsfit}{\encodingdefault}{\sfdefault}{m}{sl}
\SetMathAlphabet{\mathsfit}{bold}{\encodingdefault}{\sfdefault}{bx}{n}

\newcommand{\DCBM}{\texttt{DCBM}}
\newcommand{\DCBMNT}{\texttt{DCBM-NT}}
\newcommand{\DCBMNC}{\texttt{DCBM-NC}}
\newcommand{\DBB}{\texttt{DBB}}
\newcommand{\BB}{\texttt{BB}}
\newcommand{\CBM}{\texttt{CBM}}

\DeclareMathOperator*{\argmax}{arg\,max}
\DeclareMathOperator*{\argmin}{arg\,min}

\usepackage{mathtools}
\usepackage{bm}
\usepackage{tcolorbox}
\usepackage{booktabs}
\usepackage{ifthen}
\usepackage{amsfonts}
\usepackage{amssymb}
\usepackage{xparse}
\usepackage{comment}

\newcounter{issuecount}
\newcounter{todocount}
\NewDocumentCommand{\inlinetodo}{o m}{%
    \stepcounter{todocount}%
    \textcolor{orange}{%
        \textbf{{Todo}
            \arabic{todocount}%
            \IfNoValueTF{#1}{.}{ (#1).}}} {\textcolor{orange!80!white}{#2}}
}
\newcounter{notecount}
\NewDocumentCommand{\note}{o m}{%
    \stepcounter{notecount}%
    \par\noindent%
    \textcolor{blue}{%
        \textbf{{Note}
            \arabic{notecount}%
            \IfNoValueTF{#1}{.}{ (#1).}}} {\textcolor{blue!80!white}{#2}}
}

\renewcommand{\vec}[1]{\ensuremath\bm{{#1}}}

\newcommand{\subscm}[3][]{
    \ifthenelse{\equal{#1}{}}
    {\ensuremath{\mathcal{#2}\left[#3\right]}}
    {\ensuremath{\mathcal{#2}^{#1}\left[#3\right]}}
}

\newcommand{\real}{\ensuremath\mathbb{R}}

\newcommand{\set}[1]{\ensuremath\bm{{#1}}}

\newcommand{\setdef}[1]{\ensuremath\left\{{#1}\right\}}
\newcommand{\dom}[1]{\ensuremath\mathcal{D} (#1)}

\begin{document}

\twocolumn[
\icmltitle{%
Deferring Concept Bottleneck Models:\\
Learning to Defer Interventions to Inaccurate Experts
}



\icmlsetsymbol{equal}{*}

\begin{icmlauthorlist}
\icmlauthor{Andrea Pugnana}{1}
\icmlauthor{Riccardo Massidda}{1}
\icmlauthor{Francesco Giannini}{2}
\icmlauthor{Pietro Barbiero}{3}
\icmlauthor{Mateo Espinosa Zarlenga}{4}
\icmlauthor{Roberto Pellungrini}{2}
\icmlauthor{Gabriele Dominici}{3}
\icmlauthor{Fosca Giannotti}{2}
\icmlauthor{Davide Bacciu}{1}
\end{icmlauthorlist}

\icmlaffiliation{1}{University of Pisa}
\icmlaffiliation{2}{Scuola Normale Superiore}
\icmlaffiliation{3}{Universita della Svizzera Italiana}
\icmlaffiliation{4}{University of Cambridge}

\icmlcorrespondingauthor{Andrea Pugnana}{andrea.pugnana@di.unipi.it}

\icmlkeywords{Machine Learning, ICML}

\vskip 0.3in
]



\printAffiliationsAndNotice{}  

\begin{abstract}
Concept Bottleneck Models (CBMs) are machine learning models that improve interpretability by grounding their predictions on human-understandable concepts, allowing for targeted interventions in their decision-making process. However, when intervened on, CBMs assume the availability of humans that can identify the need to intervene and always provide correct interventions. Both assumptions are unrealistic and impractical, considering labor costs and human error-proneness. In contrast, Learning to Defer (L2D) extends supervised learning by allowing machine learning models to identify cases where a human is more likely to be correct than the model, thus leading to deferring systems with improved performance. In this work, we gain inspiration from L2D and propose Deferring CBMs (DCBMs), a novel framework that allows CBMs to learn when an intervention is needed. To this end, we model DCBMs as a composition of deferring systems and derive a consistent L2D loss to train them. Moreover, by relying on a CBM architecture, DCBMs can explain why defer occurs on the final task. Our results show that DCBMs achieve high predictive performance and interpretability at the cost of deferring more to humans.
\end{abstract}

\section{Introduction}%
\label{sec:introduction}

\begin{figure*}[t]
\centering
\includegraphics[width=0.98\textwidth]{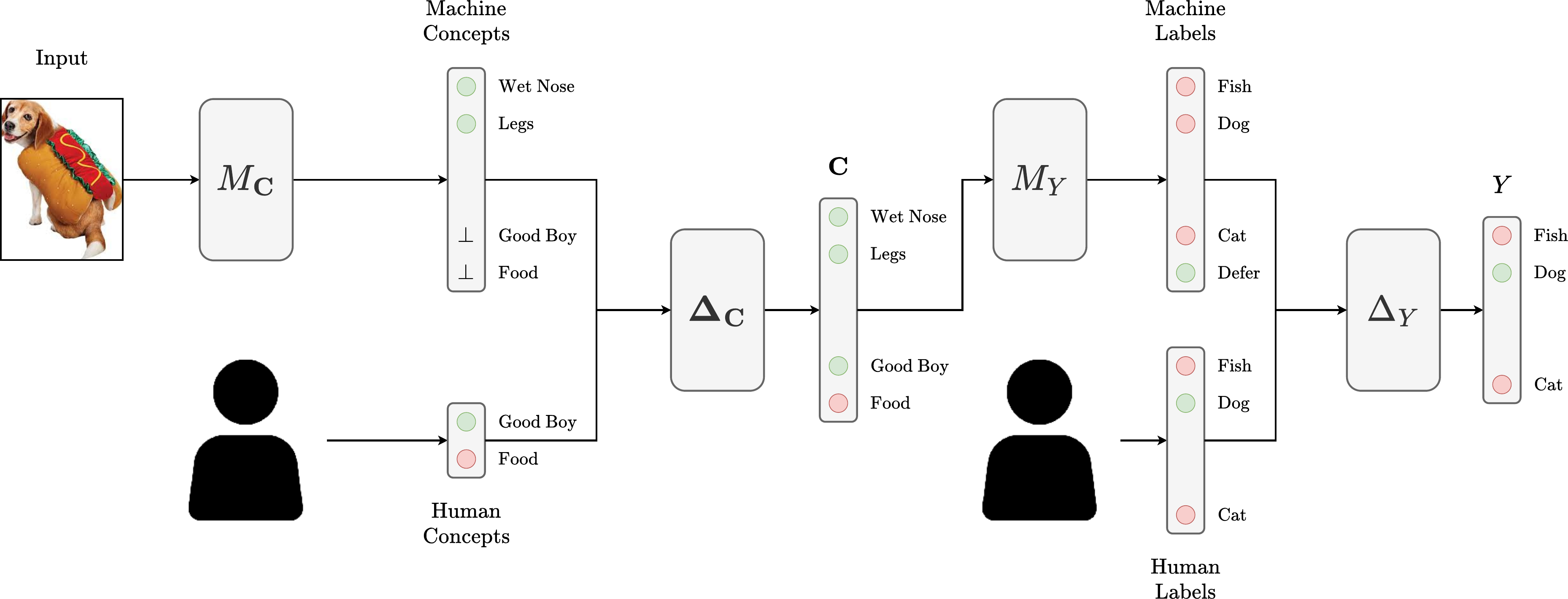}
\caption{%
    A DCBM:
        Given an input, the concept predictors $M_{\set{C}}$ output either a concept's value or defer its prediction to a human (i.e., they predict $\bot$). Next, the deferring system $\Delta_{\set{C}}$ outputs the human labels \emph{only} on the deferred concepts, returning the system's predictions otherwise. The same applies to the final task, where the task classifier $M_Y$ is an input of a dedicated deferring system~$\Delta_Y$. DCBMs can be trained by considering the cost of deferring, thus regulating the expected number of human deferrals.
}
\label{fig:visual_abstract}
\end{figure*}

Concept Bottleneck Models (CBMs) \cite{koh2020concept} are a family of interpretable machine learning (ML) models that incorporate human-interpretable \textit{concepts} 
as part of their training and predictive process.
At test time, CBMs enable experts to correct any of their concepts' values, potentially triggering a change to the CBM's task prediction.
This fosters a collaborative interaction between humans and AI systems, where a CBM may improve its accuracy when deployed with the support of an expert. 
%
However, CBMs suffer from a few shortcomings: first, increasing interpretability often comes at the expense of predictive accuracy, leading to an interpretability-accuracy trade-off~\cite{DBLP:conf/nips/ZarlengaBCMGDSP22}; second, CBMs often assume that their set of concepts can fully predict the final task (i.e., they are \textit{complete}~\citep{yeh2020completeness}); third, CBMs assume that human interventions are \textit{infallible}, an over-simplification that does not reflect the real world where human experts may introduce errors, be unaware of their own potential weaknesses, and have a specific sub-expertise~\citep{DBLP:journals/pacmhci/RastogiZWVDT22}. These practical limitations muddle the effects and dynamics of the human-AI collaboration expected when using CBMs.

To address the complex dynamics of human-in-the-loop interactions, Learning to Defer (L2D) has been introduced as an extension of supervised learning~\citep{madras2018predict,DBLP:conf/nips/OkatiDG21,DBLP:conf/icml/MozannarS20}. In L2D, ML models can delegate challenging instances to human experts, enhancing human-AI team collaboration and outperforming both the ML models and the human experts~\citep{DBLP:conf/aistats/MozannarLWSDS23}.
Notably, conventional L2D approaches have been applied to single-classification tasks and are typically opaque, providing little insight into the reasons for deferring decisions.


In this work, we introduce \emph{Deferring Concept Bottleneck Models} (DCBMs), a novel class of models enabling learning to defer on CBMs (\Cref{fig:visual_abstract}). A key advantage of DCBMs is their ability to effectively learn when a concept or task prediction could benefit from human intervention.  To the best of our knowledge, DCBM represents the first interpretable-by-design deferring system, enabling more transparent human-AI collaborations. 
Moreover, resorting to L2D, DCBMs introduce another variable to the accuracy-interpretability trade-off, i.e. the so-called \textit{coverage}, which measures the percentage of times the ML model provides the prediction. Indeed, by allowing DCBMs to defer difficult cases to the human, one can achieve high accuracy and interpretability at the cost of deferring more to the human.
%
Summarizing, our contributions are the following:
\begin{enumerate}[itemsep=-3pt,leftmargin=*]
    \item%
    We introduce the Deferring Concept Bottleneck Model, a new CBM capable of deferring on both its intermediate concepts and final task predictions (\cref{subsec:dcbm}).
    \item%
    We propose a new deferral-aware loss for CBMs (\Cref{subsec:maxlike}) and prove that it is a consistent surrogate loss w.r.t.\ the intractable zero-one loss on the deferral procedure (\Cref{subsec:consistency}).
    \item
    We experimentally show how DCBMs react to varying costs and different human-accuracy degrees for defer~(\Cref{sec:experiments}). Moreover,  deferring on CBMs can significantly improve concept-incomplete tasks.
    \item
    Finally, we demonstrate how DCBMs can produce concept-based explanations for their final task deferrals by exploiting their interpretable-by-design architecture.
\end{enumerate}

The rest of the paper is organized as follows. In \Cref{sec:background}, we introduce the background on CBMs and L2D.  Then, in \Cref{sec:method}, we propose DCBMs and prove that their loss function is consistent with existing L2D formalizations. Next, in \Cref{sec:experiments}, we report an empirical analysis highlighting the advantages of DCBMs. Finally, we discuss related works in \Cref{sec:related} and summarize our work in \Cref{sec:conclusion}.

\section{Background}\label{sec:background}

Given a variable $V$, we denote its domain as $\dom{V}$, and its realization as $v\in\dom{V}$. Similarly, we use bold for sets of variables and their multi-variate realizations.

\textbf{Concept Bottleneck Models.}
Concept-based models are interpretable architectures that explain their predictions using high-level units of information known as ``concepts''~\cite{kim2018interpretability, chen2020concept,DBLP:conf/icml/KimJPKY23,barbiero2023interpretable,DBLP:conf/iclr/OikarinenDNW23}. Most of these approaches can be formulated as a Concept Bottleneck Model (CBM)~\citep{koh2020concept}, an architecture where predictions are made by composing (i) a \textit{concept encoder} $g: \dom{\set{X}} \to \dom{\set{C}}$ that maps samples $\vec{x} \in \dom{\set{X}} \subseteq \mathbb{R}^d$ (e.g., pixels) to a set of $n_c$ concepts $\vec{c} \in \dom{\set{C}} = \{0,1\}^{n_c}$ (e.g., ``\texttt{red}'', ``\texttt{round}''), and (ii) a \textit{task predictor} $f: \dom{\set{C}} \to \dom{\set{Y}}$ that maps predicted concepts to a set of $n_y$ tasks $\vec{y} \in \dom{\set{Y}} = \{0,1\}^{n_y}$ (e.g., ``\texttt{apple}'', ``\texttt{pear}'').

CBMs can be trained (a) \textit{independently}, where $g$ and $f$ are trained separately and later combined; (b) \textit{sequentially}, where $g$ is trained first, and its output is used to train $f$; or (c) \textit{jointly}, where $g$ and $f$ are trained together.
All of these training paradigms operate under the assumption that the training concept labels $\vec{c}$ are  \textit{complete}, meaning they are sufficient to predict the tasks $\vec{y}$~\citep{yeh2020completeness}.


\textbf{Learning to Defer.} 
Learning to Defer (L2D)~\citep{madras2018predict}
combines a human expert's knowledge, 
modeled as a \textit{given}
and
non-trainable
predictor
$h\colon\dom{\set{X}}\to\dom{\set{Y}}$, together with a \textit{learnable}
classifier~$m\colon\dom{\set{X}}\to\dom{\set{Y}}$.
In combination with a selector function $s \colon \dom{\set{X}} \to \{0,1\}$, which decides whether to trust the classifier's prediction (i.e., $s(\vec{x})=0$) or the human's prediction (i.e., $s(\vec{x})=1$), these functions form a \textit{deferring system} $\Delta: \dom{\set{X}} \to \dom{\set{Y}}$
\[
    \Delta(\vec{x}) =
    \begin{cases}
        m(\vec{x}) &\text{if }s(\vec{x})=0\\
        h({\vec{x}}) &\text{otherwise.}
    \end{cases}
\]

A deferring system is a human-AI team specifying who should predict
between the human and the ML model.
We stress here that
the human predictions might \emph{differ}
from the ground-truth label,
and thus trivially deferring
each instance might not be optimal.
According to \citet{DBLP:conf/icml/MozannarS20}, L2D can be formalized as a risk minimization problem of the following zero-one loss:
\begin{equation}
\label{eq:L2D_min}
\begin{split}
\min_{m \in \mathcal{M},\,s \in \mathcal{S}} \mathbb{E}_{\vec{x}, y, h\sim\mathbb{P}(\mathbf{x}, y, h)}[\mathbb{I}_{\{s(\vec{x})=0\}}\mathbb{I}_{\{m(\vec{x}) \neq y\}} \\+ \mathbb{I}_{\{s(\vec{x})=1\}}\mathbb{I}_{\{h \neq y\}}],
\end{split}
\end{equation}
where $\mathbb{P}(\mathbf{x}, y, h)$ is the distribution over (input, output, human-predictions) triplets,
and $\mathcal{M}$ and $\mathcal{S}$ are the hypothesis spaces for the model~$m$ and the selector function~$s$.

Since directly optimizing \Cref{eq:L2D_min} is intractable, many \emph{consistent surrogate losses}\footnote{Let $\ell$ and $\ell'$ be two loss functions. $\ell'$ is a consistent surrogate of $\ell$ whenever $\argmin \ell' \subseteq \argmin \ell$.} have been proposed~\citep{DBLP:conf/icml/MozannarS20,DBLP:conf/icml/VermaN22,DBLP:conf/nips/CaoM0W023,DBLP:conf/icml/CharusaieMSS22} to train single-task classifiers over $|\set{Y}|+1$ classes, where the additional class, denoted as $\bot$, stands for the \textit{deferral decision}. Hence, we equivalently refer to deferring systems as pairs $\Delta=(m, h)$, where $\Delta(\vec{x})=m(\vec{x})$ if $m(\vec{x})\neq\bot$ and $\Delta(\vec{x})=h(\vec{x})$ otherwise.


\section{Deferring Concept Bottleneck Models}%
\label{sec:method}


The section is organized as follows. We first formulate the exact learning problem for DCBMs (\Cref{subsec:dcbm}). Then we introduce the surrogate loss and we show how it can be derived as the maximum likelihood of our graphical model (\Cref{subsec:maxlike}). 
Finally, we prove that our formulation results in a valid consistent loss for the L2D problem (\Cref{subsec:consistency}), and we discuss how the loss' consistency can be ensured, while efficiently training  DCBMs (\Cref{subsec:pretrain}).

\subsection{Model Formulation}%
\label{subsec:dcbm}

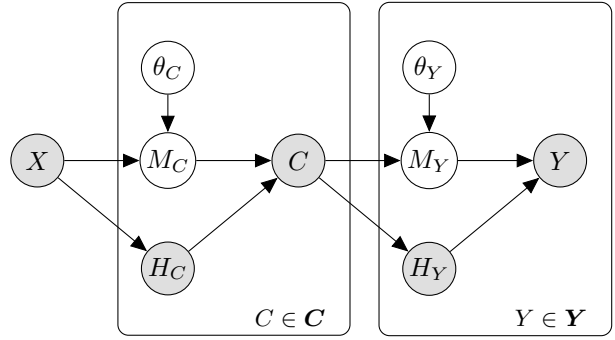
\begin{figure}
\centering
\begin{tikzpicture}[->]
     \node[obs] (x) {$X$};
     \node[latent,right=of x] (ModelCi) {$M_{C}$}; %
     \node[latent,above=of ModelCi, yshift=-0.5cm] (ThetaCi) {$\theta_{C}$}; %
     \node[obs,below=of ModelCi, yshift=0.3cm] (HumanCi) {$H_{C}$}; %
     \node[obs,right=of ModelCi] (Ci) {${C}$}; %
     \node[latent,right=of Ci] (ModelYj) {$M_{Y}$}; %
     \node[latent,above=of ModelYj, yshift=-0.5cm] (ThetaYj) {$\theta_{Y}$}; %
     \node[obs,below=of ModelYj, yshift=0.3cm] (HumanYj) {$H_{Y}$}; %
     \node[obs,right=of ModelYj] (Yj) {${Y}$}; %
     \plate [inner sep=.3cm,xshift=.02cm,yshift=.2cm] {plate1} {(ModelCi) (ThetaCi) (HumanCi) (Ci)} {$C\in\set{C}$}; %
     \plate [inner sep=.3cm,xshift=.02cm,yshift=.2cm] {plate1} {(ModelYj) (ThetaYj) (HumanYj) (Yj)} {$Y\in\set{Y}$}; %
     \edge {x,ThetaCi} {ModelCi}
     \edge {x} {HumanCi}
     \edge {Ci} {HumanYj}
     \edge {ModelCi,HumanCi} {Ci}
     \edge {Ci,ThetaYj} {ModelYj}
     \edge {ModelYj,HumanYj} {Yj}
\end{tikzpicture}
 \caption{%
   A DCBM is a Bayesian Network where inputs~$\set{X}$, concepts~$\set{C}$, tasks~$\set{Y}$, and human labels~$\set{H}$ are observed variables (in \textcolor{gray}{gray}). As represented by the plate notation~\citep{koller2009probabilistic}, we assign a human expert and a latent model to each variable.
   We incorporate the deferral decision in the model through a dedicated output, denoted as $M=\bot$. Here, we learn each model $M_V$'s parameters  $\theta_V$ via maximum likelihood.
}\label{fig:cbndefer}
\end{figure}

As in CBMs, we consider the problem of predicting both concept variables that directly depend on the input and task variables that are conditionally independent of the input given the concepts. We define DCBMs as an extension of CBMs, where each concept variable $C \in \set{C}$ and each task variable $Y \in \set{Y}$ is dealt with as a separate deferring system. Similar to CBMs, a DCBM can be framed as a probabilistic graphical model, with the difference that both concepts and tasks depend on human predictions when the model defers (Figure~\ref{fig:cbndefer}).

Let the set of concepts and task variables be $\set{V}=\set{C} \cup \set{Y}$. We assign an expert~$H_V$ and a model~$M_V \colon \dom{\set{Z}_V} \to [n_V] \cup \setdef{\bot}$ to each variable $V \in \set{V}$. Here, the output space consists of $n_V + 1$ classes, including the deferral choice~$\bot$, and either $\set{Z}_V = \set{X}$, if $V \in \set{C}$, or $\set{Z}_V = \set{C}$, if $V\in\set{Y}$. Similarly, we denote the ground-truth output as $k_V$, which is to be intended as the label of either a concept or a task.

In contrast to traditional L2D setups, in a DCBM we need to train a model composed of several deferring systems, one for each concept and task variable. Hence, our objective would ideally minimize the number of mistakes made by all the deferring systems. This can be expressed through the following multi-variate zero-one loss, where we account for the cost of each deferral via a hyperparameter $\lambda \in [0, 1]$:
\begin{definition}[Multivariate Zero-One Loss]
\label{def:multivar_loss}
    Given a set of deferring systems $\set{\Delta} = \setdef{\Delta_V =(m_V, h_V)}_{V\in\set{V}}$
    parameterized by a set of parameters $\set{\theta}=\setdef{\theta_V}_{V\in\set{V}}$
    over a set of variables $\set{V}$ 
    , we define the multivariate zero-one loss as
    \begin{equation}\label{eq:zeroone}
    \begin{split}
         \sum_{V\in\set{V}}
         &\mathbb{I}_{\{m_V(\vec{z}_V)\neq\bot\}}
         \mathbb{I}_{\{m_V(\vec{z}_V) \neq k_V\}}
         \\+
         &\mathbb{I}_{\{m_V(\vec{z}_V)=\bot\}}
         (\lambda + 
         \mathbb{I}_{\{h_V\neq k_V\}}
         )
         \ ,
    \end{split}
    \end{equation}
    where $\vec{z}_V$ and $k_V$ are the realizations of $\Delta_V$'s inputs and outputs, respectively.
\end{definition}

\subsection{Maximum Likelihood and Surrogate Loss}
\label{subsec:maxlike}
By deriving the negative log-likelihood from our probabilistic formulation of DCBMs (\Cref{fig:cbndefer}), 
we can treat the maximum likelihood estimation of the parameters as a minimization problem. In this way, we obtain a loss function composed of two terms. Intuitively, the first term directly rewards the classifier for predicting the ground-truth class, while the second term rewards the model for deferring whenever the human prediction is correct. 
\begin{proposition}[%
Maximum Likelihood of DCBM]
\label{prop:maxlikeDCBM}
    Let $\set{\theta}$ be the parameters
    of a DCBM. Then,
    we can obtain the most
    likely parameters $\hat{\theta}$
    given observations on the inputs~$\vec{x}$,
    the concepts~$\vec{c}$, the human~$\vec{h}$,
    and the task~$\vec{y}$,
    by minimizing the following loss function:
    \begin{equation}
    \begin{split}
    \ell(\set{\theta} \mid \vec{x},\vec{c}, \vec{y}, \vec{h})
    &=
    \sum_{V \in\set{V}}\Big(
    \Psi\big(q(\vec{z}_V; \theta_V), k_V\big)\\
    &\quad+
    \mathbb{I}_{\left\{y_V = h_V\right\}}
    \Psi\big(q(\vec{z}_V; \theta_V), \bot\big)\Big)\\
    \end{split}
    \label{eq:psiloss}
    \end{equation}
    where $q(\hspace{0.3em}\cdot\hspace{0.3em};\theta_V)\colon\dom{\set{Z}_V}\to\real^{K_V+1}$
    returns the logits of the model~$M_V$
    given $\vec{z}_V\in\dom{\set{Z}_V}$
    and $\Psi(q(\vec{z}_V;\theta_V), k)$
    is the negative log-probability of the class~$k$ given the logits.
\end{proposition}
\begin{proof}
    We report the proof in \Cref{app:proofloss}.
\end{proof}

When $\Psi$ corresponds to the cross-entropy formulation
\[
\Psi(q(\vec{z}), k) = -\log\left(\frac{\exp(q(\vec{z})_k)}{\sum_{k'\in[K+1]}\exp(q(\vec{z})_{k'})}\right),
\]
the negative log-likelihood in Equation~\ref{eq:psiloss} corresponds to the sum over all variables $V\in\set{V}$ of a known function from the L2D literature~\citep[Eq. 6]{DBLP:conf/icml/MozannarS20}. In this way, we first establish a connection between the maximum likelihood problem and the learning to defer task that, to the best of our knowledge, has not been previously identified in the literature. Depending on the chosen parameterization of $\Psi$ (see \Cref{tab:liu} in \Cref{app:constrained} for some alternatives), the same formulation 
can derive several known L2D losses.

The negative log-likelihood in Equation~\ref{eq:psiloss} of the DCBM does not take into account the cost of deferring. In this way, in scenarios where the human has a significant advantage, we can expect the model to underfit and almost always defer to the human~\citep{DBLP:conf/aistats/MozannarLWSDS23}. To overcome these limitations, we define a penalized loss function by constraining the parameters of the model to enforce two additional conditions: (1) the model should not always defer when the human is correct, and (2) when the human is not correct, the model should not defer. We report the formalization of the constrained optimization problem in \Cref{app:constrained}, and hereby report the resulting penalized loss function,
\begin{equation}
\begin{split}
\ell(\set{\theta} \mid \vec{x},\vec{c},& \vec{y}, \vec{h})
\\
&=\sum_{V \in\set{V}}
\Psi(q(\vec{z}_V;\theta_V), v)\\
&+
(1-\lambda) \cdot \mathbb{I}_{\{y_V = h_V\}}
\Psi(q(\vec{z}_V; \theta_V), \bot)\\
&+
\lambda \cdot \mathbb{I}_{\{y_V \neq h_V\}}
\sum_{k\in[K]}
\Psi(q(\vec{z}_V; \theta_V), k),
\end{split}
\label{eq:psilosspenalized}
\end{equation}
where $\lambda\in[0,1]$ is the hyperparameter
that controls the trade-off between the deferring
to the human and the machine learning model.
Even in this case, the resulting penalized loss in \Cref{eq:psilosspenalized} coincides with the sum over variables of a known single-variate learning to defer loss function~\citep{DBLP:conf/aistats/LiuCZF024}. 

\subsection{Loss Consistency}%
\label{subsec:consistency}

The multivariate scenario exacerbates the fact that always deferring to a human might not be the right choice, as the human's feedback may be incorrect or costly. Therefore, to ensure that our model effectively defers to the human only when needed, we have to prove that the cost-free loss (Equation~\ref{eq:psiloss}) and the penalized loss (Equation~\ref{eq:psilosspenalized}) of a DCBM are consistent surrogates of the multivariate zero-one loss (Equation~\ref{eq:zeroone}). We prove this by first showing that the sum of consistent losses on deferring systems with different parameters is a consistent loss for the whole system.
\begin{lemma}
\label{lemma:sumconsistent}
Let $\ell^\prime_1, \ell_1, \cdots,\ell^\prime_m, \ell_m$  be (possibly distinct) loss functions. Assume that, for every $i\in\{1,\ldots,m\}$, $\ell^\prime_i,\ell_i:\mathbb{R}^{n_i}\to\mathbb{R}$, being $\ell'_i$ a consistent surrogate of $\ell_i$.
Then $\ell^\prime:\mathbb{R}^{n}\to\mathbb{R}$, with $n=n_1+\ldots+n_m$ and 
$\ell'(\theta_1,\ldots,\theta_m)=\sum_{i=1}^m\ell^\prime_i(\theta_i)$
     is a consistent surrogate of
$\ell:\mathbb{R}^{n}\to\mathbb{R}$, with $\ell(\theta_1,\ldots,\theta_m)=\sum_{i=1}^m\ell(\theta_i)$.
\end{lemma}
%
\begin{proof}
    We report the proof in \Cref{app:proofsum}.
\end{proof}
\begin{theorem}
\label{theo:main}
    The cost-free loss in \Cref{eq:psiloss} and the DCBM penalized loss in \Cref{eq:psilosspenalized} are surrogate consistent losses of the multivariate zero-one loss of \Cref{eq:zeroone} when $\set{V}=\set{C}\cup\set{Y}$, and $\lambda=0$ and $\lambda>0$, respectively.
\end{theorem}
\begin{proof}
    We report the proof in Appendix~\ref{app:proofmain}.
\end{proof}

Hence, minimizing our introduced surrogate losses, under appropriate assumptions,
(whose practicalities are discussed in the following subsection), 
corresponds to minimizing the zero-one loss associated to the exact problem. 

\subsection{Consistent Training of DCBMs}%
\label{subsec:pretrain}

\Cref{theo:main} ensures the consistency of our overall formulation under specific assumptions on the loss functions to be summed together: they should depend on disjoint sets of parameters. Hence from a practical point of view, there are two main requirements to ensure its consistency while training a DCBM.  
First, the model has to be trained \emph{independently}, so that no information flows from the tasks' losses to the concepts' losses. 
Notably, independent training of CBMs is known to slightly decrease the performance compared to \emph{jointly} training CBMs~\cite{koh2020concept}. However, independent training avoids the problem of concept leakage~\cite{mahinpei2021promises, havasi2022addressing}, inherent to jointly trained models, thus maintaining the interpretability of the outcomes. For this reason, we focus on independently trained models here and discuss additional experiments on jointly trained models, showing similar results to those seen for their independent counterparts, in \cref{app:additionalresults}.
%

Our second requirement is that the concept predictors should not share their parameters in the DCBM's architecture, and similarly for the task predictors.
Parameter sharing is common in CBMs, especially for computer vision tasks~\citep{DBLP:conf/nips/ZarlengaBCMGDSP22}, where an encoder produces a latent representation from the input space that is then fed to the concept predictors. To enable this in applications where parameter sharing is beneficial, we take the following two-step approach: first, we train an encoder to predict either all the concepts or the final task from the input features. Then, we freeze this encoder, discard the learned predictors, and independently train the concept predictors on the encoder's latent representation and the task predictor on the concepts using our consistent L2D loss.
Still, for completeness' sake, we evaluate DCBMs when they share parameters across classifiers in \Cref{app:additionalresults}.

\section{Experimental Evaluation}%
\label{sec:experiments}

Our experiments aim to answer these research questions\footnote{%
We provide the code for reproducing our empirical analysis at~\textcolor{blue}{\url{https://anonymous.4open.science/r/DCBM-03E2/README.md}}.
}:
\begin{itemize}
    \item [\textbf{Q1}]%
    Does the option to defer improve the performance of independently trained CBM-based approaches? 
    \item [\textbf{Q2}]%
    Does the option to defer mitigate the lack of completeness of a set of concepts for task prediction accuracy?
    \item [\textbf{Q3}]%
    How does human competence affect deferring?
    \item [\textbf{Q4}]%
    Can concepts explain \textit{why} a final task was deferred?

\end{itemize}

\subsection{Experimental Settings}

\textbf{Datasets.}
We perform our analysis on three datasets: \texttt{completeness} \citep{laguna2024beyond}, \texttt{cifar10-h}~\citep{Peterson2019HumanUM} and \texttt{CUB}~\citep{cub}.

The \texttt{completeness} dataset is a synthetic dataset that allows full control of the data-generation process. Here, we simulate labels from human experts with different competencies by selecting the concepts' or task's correct labels with different probabilities. In particular, we denote as \texttt{oracle}, \texttt{human-80\%} and \texttt{human-60\%}, a human that correctly predicts their labels with an accuracy of 100\%, 80\% and 60\%, respectively.
The \texttt{cifar10-h} dataset is a modified version of the \texttt{cifar10} dataset \cite{cifar10} containing 10,000 images with both ground-truth and human-annotated labels. 
We adapted it for our scenario by adding annotated concepts. We employed the 16 ``superclass'' concepts defined by \citet{DBLP:conf/iclr/OikarinenDNW23} for each class. As human annotations are missing for the concepts, we treat humans as oracles on the concepts.
%
Finally, \texttt{CUB} is a dataset commonly used for image classification with CBMs. We consider the complete set of 112 concepts used by \citet{koh2020concept}. Since the dataset reports annotator uncertainty on the concepts, we use them to produce random human concept labels as done by~\citet{human_uncertainty_cbms}.
In the \texttt{CUB} task label, however, we treat humans as oracles.

\textbf{Baselines and Variants.}
We compare a complete DCBM architecture (\DCBM) with some ablated variants and baselines. In particular, we consider the following baselines:
(i) a black-box model trained with standard supervised learning on the final task only (\BB);
(ii) a black-box model that can defer on the final task only (\DBB); 
(iii) a standard CBM without the deferring option (\CBM). 
To evaluate the effect of deferring on concepts and tasks of \DCBM{}, we also compare with the following ablations: a DCBM that can not defer on the final task (\DCBMNT) and a DCBM that can not defer on the concepts (\DCBMNC). In all the datasets, we train the models using the state-of-the-art Asymmetric Softmax~\citep[ASM]{DBLP:conf/nips/CaoM0W023} parameterization of the negative-log-likelihood in our loss functions. We provide further details on the adopted architectures and the experimental setup in \Cref{app:exp_det}. We report in \Cref{app:additionalresults} results for other losses on the \texttt{completeness} dataset.

\textbf{Metrics.}
Given a sample of size $N$, we consider four main metrics: the accuracy on the final task $AccTask$, the accuracy on concepts $AccConc$, the final task's coverage $CovTask = \frac{1}{N} \sum_{i=1}^{N} \mathbb{I}_{\{m_Y(\vec{c})\neq\bot\}}$, which measures how many times the ML model provides the task prediction, and the average concepts coverage $CovConc= \frac{1}{n_c}\sum_{c=1}^{n_c}\left(\frac{1}{N} \sum_{i=1}^{N} \mathbb{I}_{\{m_c(\vec{x})\neq\bot\}}\right)$, which measures how many times the ML model provides concept predictions.

\subsection{Experimental Results}%
\label{subsec:expres}

\begin{figure*}[t]
    \centering
    \includegraphics[width=\linewidth]{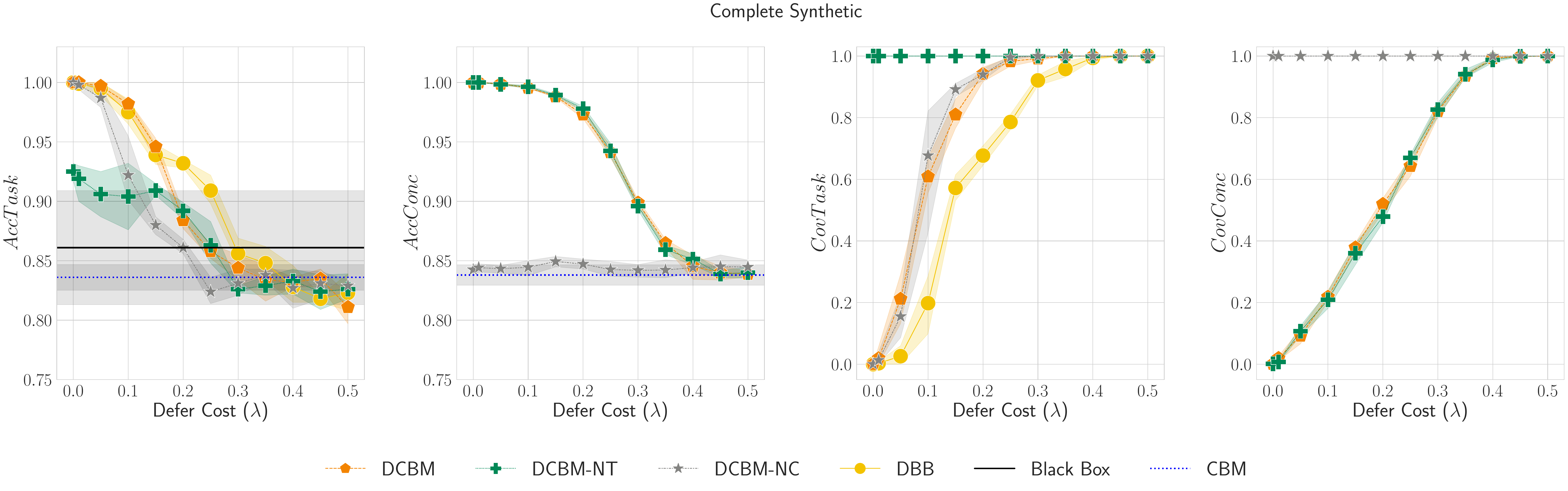}
    \caption{%
    Results on \texttt{completeness} when human experts have perfect concept and task accuracy (i.e., they are oracles).
    We report each metric's average and standard deviations as we increase the defer cost~$\lambda$.
    The black box and the CBM baselines are constant as they are independent of the defer cost.
    }
    \label{fig:COMPLETEres}
\end{figure*}

\begin{figure*}[t]
    \centering
    \includegraphics[width=\linewidth]{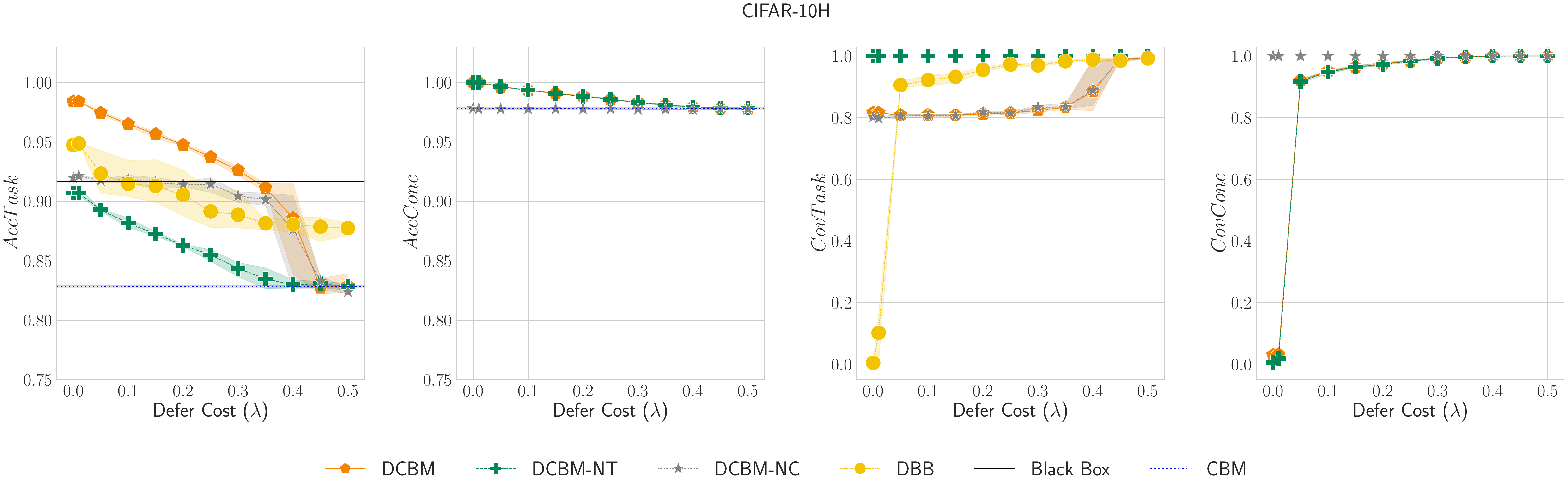}
    \caption{
    Results on \texttt{cifar10-h}  when human experts have perfect accuracy on the final task but not on the concepts.
    We report each metric's average and standard deviation as we increase the defer costs~$\lambda$.
    The black box and the CBM baselines are constant as they are independent of the defer cost.
    }
    \label{fig:CIFAR10Hres}
\end{figure*}

\begin{figure*}[t]
    \centering
    \includegraphics[width=\linewidth]{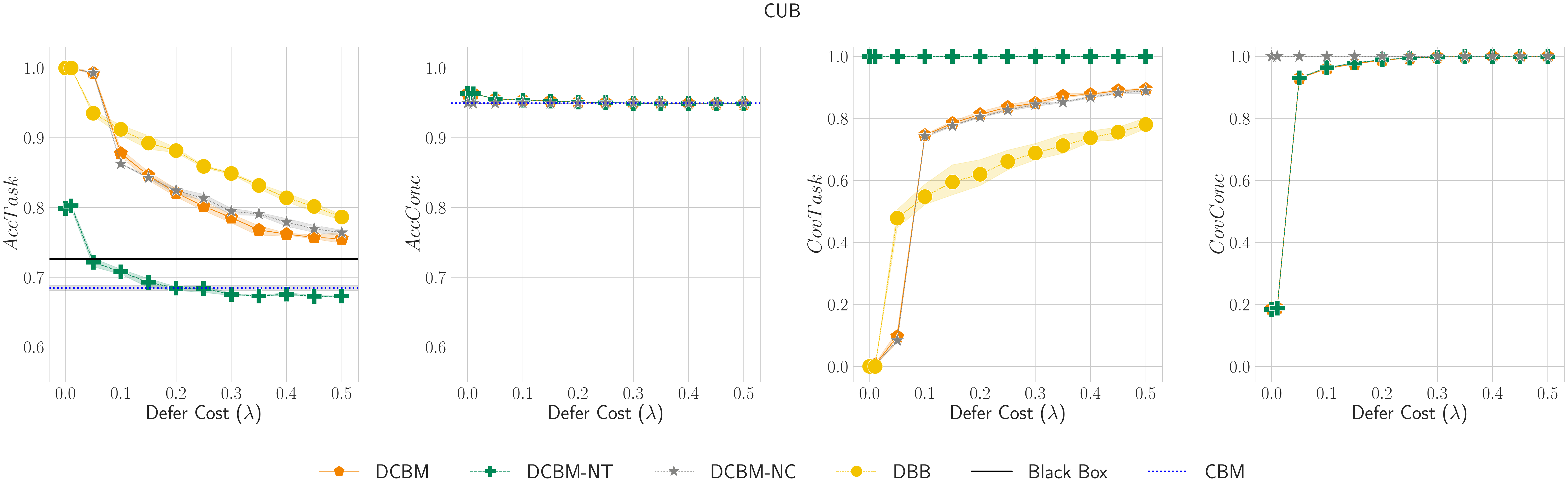}
    \caption{
    Results on \texttt{CUB} dataset when human experts have perfect accuracy on the concepts but not on the final task. We report each metric's average and standard deviation as we increase the defer costs~$\lambda$.
    The black box and the CBM baselines are constant as they are independent of the deferring cost.
    }
    \label{fig:CUBres}
\end{figure*}

\textbf{Q1: CBM Improvements by Deferring.}
We study the performance improvement on CBMs in a controlled environment by employing the synthetic \texttt{completeness} data (\Cref{fig:COMPLETEres}). When the defer costs are significantly low ($\lambda < 0.05$), the deferring system tends to over-rely on the human and thus the coverage of the machine learning model is zero ($CovTask=0, AccTask=1$). As expected, increasing the defer cost increases both $CovTask$ and $CovConc$. At the same time, it also reduces the accuracy of the prediction, which is, however, still \emph{over} the standard \CBM{} baseline without deferring capabilities. In summary, the performance of the ablated (\DCBMNC{}, \DCBMNT{}) and the full model (\DCBM{}) tend to those of the standard non-deferring \CBM{} for higher defer costs while improving performance for lower defer costs.

Notably, deferring only on the concepts (\DCBMNT{}) or on the task (\DCBMNC{}) is not sufficient to reach the accuracy on the task of the black-box with deferring baseline (\DBB{}) at parity of defer cost. 
Both \DCBMNC{} and \DCBM{} have a similar trend in terms of $AccTask$ for $\lambda \simeq 0$, whereas only \DCBM{}'s $AccTask$ is comparable with the one of \DBB{} for $\lambda \in [0.5, 0.15]$. However, it is worth remarking that the black box loses on the interpretability of the decision, which instead DCBMs and CBMs guarantee.

Finally, \DCBM{} does not over-rely on the human expert as increasing the defer cost~$\lambda$ effectively controls the coverage: \DCBM{}'s $CovTask$ ranges from approximately $0.301 \pm 0.128$ at $\lambda=0.05$ to approximately $0.833 \pm 0.02$ at $\lambda=0.15$, whereas \DBB{}'s $CovTask$ is $0.06\pm$ at $\lambda=0.05$, reaching only $0.55 \pm 0.148$ at $\lambda=0.15$. 

\textbf{Q2: Addressing  Incompleteness through Human-AI collaboration.}
We investigate the effect of deferring in a concept incomplete dataset through the \texttt{cifar10-h} dataset
(\Cref{fig:CIFAR10Hres}).
Concepts in this task cannot fully describe the downstream task labels:
For instance, the concepts do not allow to disentangle the \texttt{deer} class from the \texttt{cat} class. Hence, we expect deferring on the final task to be beneficial, allowing us to disentangle these cases.
The results validate our hypothesis by showing that for low defer costs, the \DCBM{} outperforms other baselines, with an $AccTask$ of $\approx0.984\pm0.001$ at $\lambda=0$ and a $CovTask$ of $\approx 0.816\pm0.000$, outperforming even \DBB{} ($AccTask\approx0.947\pm0.001$). Instead, when considering \DCBMNC{} and \DCBMNT{}, they 
reach an $AccTask$ of $\approx 0.919\pm0.005$ and $\approx 0.907\pm0.003$, respectively, thus showing a significant gap w.r.t. \DCBM{}.

However, when increasing the defer costs, defer is no longer viable, and we converge to the plain-vanilla CBM performance: $AccTask$ of \DCBM{} steadily drops until $\approx0.829\pm0.009$ at $\lambda=0.5$. Similar trends occur for all competitors, with \DBB{} achieving (the highest) $AccTask$ of $\approx0.878\pm0.005$ at $\lambda=0.5$.
These results highlight once again the trade-off between the cost of querying a human and the performance of \DCBM{}.

\textbf{Q3: Evaluating Potentially Wrong Humans.}
On the \texttt{CUB} dataset, we test the effect of deferring to human experts with different levels of accuracy on the concepts and being oracles on the tasks (\Cref{fig:CUBres}).
For low defer costs $(\lambda\leq 0.05)$, we notice that $AccTask\approx1$, as both \DCBM{} and \DCBMNC{} over-rely on the human expert ($CovTask$ is close to zero for $\lambda\leq 0.05$). 
Once we increase $\lambda$, we do not see significant gaps between \DCBM{} and \DCBMNC{} performance in terms of $AccTask$. There are two reasons for this behavior: first, $CovTask$ increases steadily due to the increase in defer costs: for \DCBM{} the metric ranges from 
$0.746 \pm 0.004$ 
at $\lambda=0.10$ 
to 
$0.889\pm0.009$ at $\lambda=0.50$.
Second, we notice that $AccConc$ plateaus around approximately $0.949$, negatively affecting the final task performance.

Such a behavior emphasizes how the human expert's ability to correctly predict concepts affects the final task's performance: in the presence of potentially incorrect humans on the concepts, there is no advantage in deferring concept predictions to humans, making \DCBM{} and \DCBMNC{} equivalent in terms of performance. Overall, a \DCBM{} can automatically adjust the coverage depending not only on the defer cost but on the competence of the human expert.

\begin{figure*}
    \centering
    \includegraphics[width=\textwidth]{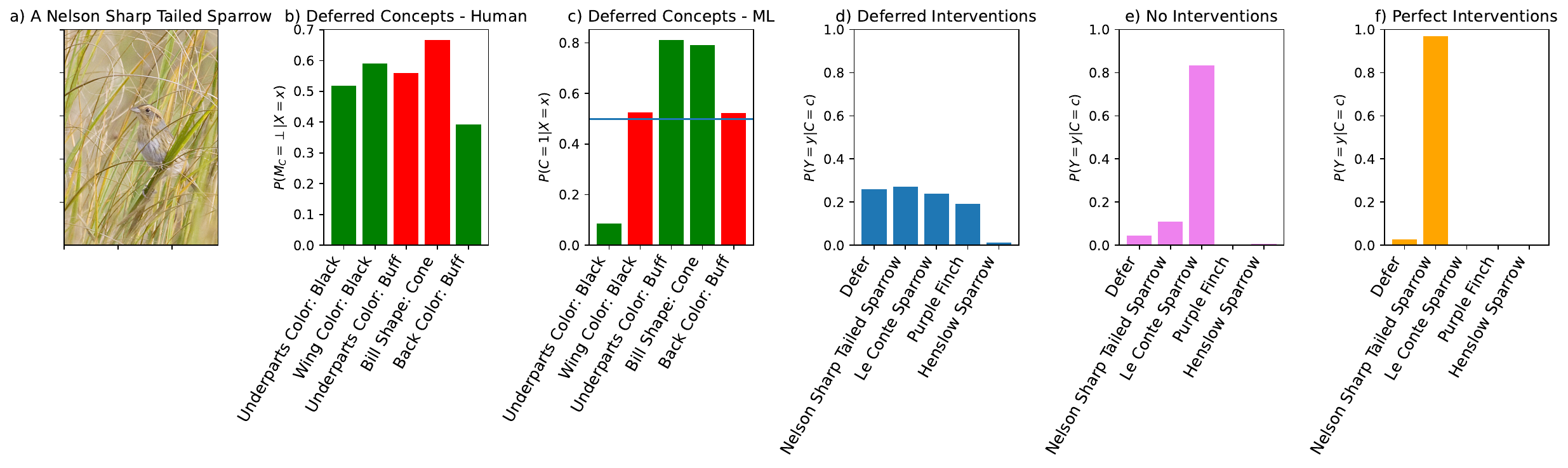}
    \caption{Interpretation of a \DCBM{} with defer cost $\lambda=0.1$ on an input sample.
    From left to right:
    (a) an example of an image from the \texttt{CUB} dataset;
    (b) the concepts that the model has deferred with the estimated probability, green bars stand for when the human correctly predicts the concept, red otherwise;
    (c) the estimated probability of each deferred concept being true according to the machine learning model, green bars stand for when the ML would have correctly predicted the concept, red otherwise;
    (d) the estimated probability of top-5 final task labels after deferring the concepts to the human (standard \DCBM{} behavior);
    (e) the estimated probability of top-5 final task labels without deferring the concepts to the human;
    (f) the estimated probability of top-5 final task labels from the ground-truth concepts;
    }%
    \label{fig:explanations}
\end{figure*}

\textbf{Q4: Interpretable Learning to Defer.}
We show how \DCBM{} can help interpret the reasons for the final task deferring, by following a similar approach to \citet{espinosa2024learning}. For this purpose, we consider the \texttt{CUB} dataset and a DCBM trained with cost $\lambda=0.1$.  In \cref{fig:explanations}, we show an example with the image of a Nelson-Sharp Tailed Sparrow. We provide other examples in \Cref{app:interpretations}.

A main advantage of \DCBM{} is its capability to identify which concepts should be corrected by a human intervention. Since the human supervision is not perfectly accurate in the \texttt{CUB} dataset, the interventions might not correspond to the ground-truth values (\cref{fig:explanations}b).
Interestingly, the deferred concepts are not easy to grasp in the original image: the underparts are partially covered by grass blades, the bill is not clearly cone-shaped, and the back is not really visible in the image.
In general, highlighting the particular concepts on which the human should be a safer option than the ML model favours the interpretation of the classifier. Moreover, we also stress that without deferring to the human, the machine learning model would have wrongly classified some deferred concepts (\cref{fig:explanations}c).

This example also stresses that \DCBM{} tries to defer only when worthy and necessary, without involving humans when they are likely to make mistakes. Moreover, interventions allows us to reason on how concepts effectively lead the \DCBM{} to predict the correct class (\cref{fig:explanations}d). In fact, without intervening on the concepts, the final task model would have predicted a wrong label (\cref{fig:explanations}e). Finally, we also report how the final task model would have behaved under perfect interventions, i.e., where a human has access to ground-truth concept labels, as assumed by the standard CBM literature (\cref{fig:explanations}f). As expected, the \DCBM{} would increase its own confidence in the correct label, at the cost of human supervision on the all of the 112 concepts instead of only the five identified by the deferring system.

\section{Related Works}\label{sec:related}

\textbf{Deferring systems.}
L2D, as introduced in~\citet{madras2018predict}, is an instance of hybrid decision-making where humans oversee machines. 
Since directly optimizing \Eqref{eq:L2D_min} is NP-hard even in simple settings~\citep{DBLP:conf/aistats/MozannarLWSDS23}, \citet{DBLP:conf/icml/MozannarS20} proposed consistent surrogate losses, which have since become the standard approach for jointly learning the deferral policy and the ML predictor~\citep{DBLP:conf/icml/CharusaieMSS22,DBLP:conf/icml/VermaN22,DBLP:conf/aistats/MozannarLWSDS23,DBLP:conf/nips/CaoM0W023,DBLP:conf/aistats/LiuCZF024,DBLP:conf/icml/WeiC024}.
A formal characterization of humans in the loop is provided by \citet{DBLP:conf/nips/OkatiDG21}. 
Recent works extend the L2D problem to account for multiple human experts, e.g.,~see~\citet{DBLP:conf/aistats/VermaBN23,mao2024two} and \citet{DBLP:conf/nips/CaoM0W023}, cases where the ML model is already given and not jointly trained, e.g.,
~\citep{DBLP:conf/icml/CharusaieMSS22} and 
\citep{mao2024two},
and how to evaluate them in a causal framework~\citep{DBLP:journals/corr/abs-2405-18902}.

\textbf{Concept Interventions.}
CBMs have seen a growth of interest in the context of \textit{concept interventions}, operations that improve a CBM's overall task performance in the presence of test-time human feedback.
Works in this area have explored (1) how to best select which concepts to intervene on next when interventions are costly~\citep{DBLP:conf/icml/ShinJAL23, coop}, (2) how to improve a model's receptiveness to interventions and learn an \textit{intervention policy}~\citep{espinosa2024learning}, and (3) how to intervene on otherwise black-box models~\citep{black_box_intervenables}.
Other approaches have exploited inter-concept relationships to propagate single-concept interventions~\citep{stochastic_cbms, inter_concept_relationships, causalconceptgraphmodels} and have used interventions as sources of continual learning labels~\citep{DBLP:conf/icml/SteinmannSFK24}.
Finally, \citet{DBLP:conf/nips/ShethK23} and \citet{human_uncertainty_cbms} both discuss notions of supervisor uncertainty, where we may be interested in modeling errors from an expert performing interventions. Nevertheless, works on concept interventions fundamentally differ from our L2D-based approach in that they assume that experts themselves trigger a correction in a model's concept predictions.
This makes it difficult for these approaches to adapt to expert-specific competencies and to be easily deployed in practice where it is desirable to know \textit{when} a human should be called to intervene.

\section{Conclusions and Future Work}%
\label{sec:conclusion}

This paper introduces DCBM, a novel approach that allows CBMs to defer to a human without additional supervision. By training the CBM with an especially designed learning to defer loss function, a DCBM can implicitly model the predictive distribution of the human, and thus defer only on instances where the expert is more likely to be correct of the machine learning model. Moreover, we formally proved the consistency of our deferring loss function for independent training of CBMs. 
Our experimental results highlight that DCBMs effectively learn when to involve a human, boosting overall predictive performance only when the human is better than the ML model. Moreover, directly involving a human helps mitigate cases where concepts are incomplete. Finally, the interpretable by-design nature of DCBMs offers ways to audit the deferring systems, showing promising results in explaining their limits.

We think that integrating CBMs with L2D have several benefits, and we envision multiple research directions for future work. First, in this paper, we consider single experts and assume that the expert's costs are the same when deferring at the concept and task level. Such an assumption might be impractical, as correcting concepts are supposed to be easier to provide than the prediction for the final task. Investigating how different costs might lead to different DCBMs is left for future work. Second, our theoretical approach considers concepts as independent variables. However, such an approach might not correctly model the relationships among concepts. Extending our approach to account for a hierarchical structure of concepts is also an open research direction.
Regarding the explanation of deferring systems, CBMs are  interpetebale models that can be used to enlighten the decision process toward either a class or a defer prediction. This paper used defer on concepts and interventions to provide explanations on task's defer. However, as several explanations' methodologies have been considered for CBMs, like the ones based on logic rules, we could extend these methods to DCBMs to explain how the predictions on the concepts have caused the defer to a human on tasks.
Finally, developing consistent learning-to-defer losses for jointly training concepts and task predictors is an open issue that needs further investigation.
\section*{Broader Impact}
Our work extends CBMs to account for potentially wrong humans. This enables one to model human behavior better, making it more suitable for real-life deployment of such systems. Moreover, the interpretable-by-design nature of CBMs positively affects the interaction between the human and the AI system, actively fostering the diffusion of such systems in high-risk, high-stakes settings.

\section*{Acknowledgements}
This work has been partially supported by the Partnership Extended PE00000013 - “FAIR - Future Artificial Intelligence Research” - Spoke 1 “Human-centered AI”. This work was also funded by the European Union under Grant Agreement no. 101120763 - TANGO.

\bibliography{bibliography}
\bibliographystyle{icml2025}

\appendix
\onecolumn
\section{Proofs}%
\label{app:proofs}

\subsection{Proposition 3.1 --- Maximum Likelihood of DCBM}%
\label{app:proofloss}

In this proof, we report the derivation of the maximum likelihood of the Bayesian Network corresponding to our Deferring Concept Bottleneck Model (DCBM), which we reported in \Cref{fig:cbndefer}.
We assume that our dataset is composed of i.i.d. samples from the joint distribution of the observable variables.
Therefore, we consider the input data~$\vec{x}\in\dom{\set{X}}$, the concept values~$\vec{c}\in\dom{\set{C}}$, the task values~$\vec{y}\in\dom{\set{Y}}$, and the human annotations on both concepts and tasks~$\vec{h}\in\dom{\set{H}}$.
We first define the likelihood of the data by marginalizing over the latent variables, i.e., of each concept model $M_C$ and task model $M_Y$ for all variables $C\in\set{C}$ and $Y\in\set{Y}$.
We recall that for each concept model $M_C$ we have one out of $n_C+1$ possible outcomes, where $n_C$ is the number of possible realizations of $C$ and the additional value accounts for the deferred decision, as in $M_C=\bot$.
Similarly, each task model $M_Y$ has $n_Y+1$ possible outcomes.
We then marginalize one variable at a time from the joint likelihood, starting from an arbitrary task variable $Y\in\set{Y}$.
\begin{align}
\mathcal{L}(\set{\theta} \mid \vec{x},\vec{c},\vec{y},\vec{h})
&=
p(\vec{x},\vec{c},\vec{y},\vec{h} \mid \set{\theta})
\\
&=
\sum_{k\in[n_Y+1]}
p(\vec{x},\vec{c},\vec{y},\vec{h}, M_Y = k \mid \set{\theta})\\
&=
p(\vec{x})
\sum_{k\in[n_Y+1]}
p(\vec{c},\vec{y},\vec{h}, M_Y = k \mid \vec{x}, \set{\theta})\\
\begin{split}
=
p(\vec{x})
\sum_{k\in[n_Y+1]}
p(Y = y \mid M_Y = k, H_Y = h_Y)
p(M_Y = k \mid \vec{c}, \theta_Y)
p(H_Y = h_Y \mid \vec{x}, \vec{c})\\
\cdot\quad
p(\vec{c},\vec{y}_{\setminus Y},\vec{h}_{\setminus Y} \mid \vec{x}, \set{\theta}_{\setminus \theta_Y}),
\end{split}\\
\begin{split}
=
p(\vec{x})
p(H_Y = h_Y \mid \vec{x}, \vec{c})
\sum_{k\in[n_Y+1]}
p(Y = y \mid M_Y = k, H_Y = h_Y)
p(M_Y = k \mid \vec{c}, \theta_Y)
\\
\cdot\quad
p(\vec{c},\vec{y}_{\setminus Y},\vec{h}_{\setminus Y} \mid \vec{x}, \set{\theta}_{\setminus \theta_Y}),
\end{split}
\end{align}
where we employ the operator $\setminus$ to denote the removal from a set of a variable.

Before marginalizing the remaining variables, we focus on the sum over possible values of the model $M_Y$. We can further decompose it by considering whether the model value is a possible value for $Y$ or a deferral~$\bot$. According to our definition of a deferring system, $Y=y$ if and only if the model has value $M_Y=y$ or it deferred the decision through $M_Y=\bot$ but for the human expert holds $H_Y=y$.
Therefore, it holds that $p(Y = y \mid M_Y = k, H_Y = h_y) = 1$ if and only if $M_Y = y$ whenever $M_Y\neq\bot$ and $p(Y = y \mid M_Y = \bot, H_Y = h_y) = 1$ if and only if $h_Y = y$ whenever $M_Y = \bot$. Formally,
\begin{align}
&\sum_{k\in[n_Y+1]}
p(Y = y \mid M_Y = k, h_Y)
p(M_Y = k \mid \vec{c}, \theta_Y)
\\
=
&\sum_{k\in[n_Y]}
p(Y = y \mid M_Y = k, h_Y)
p(M_Y = k \mid \vec{c}, \theta_Y)
+
p(Y = y \mid M_Y = \bot, h_Y)
p(M_Y = \bot \mid \vec{c}, \theta_Y)\\
=
&\sum_{k\in[n_Y]}
\mathbb{I}[y = k]
p(M_Y = k \mid \vec{c}, \theta_Y)
+
\mathbb{I}[h_Y = y]
p(M_Y = \bot \mid \vec{c}, \theta_Y)\\
=
&p(M_Y = y \mid \vec{c}, \theta_Y)
+
\mathbb{I}[h_Y = y]
p(M_Y = \bot \mid \vec{c}, \theta_Y)
\end{align}
where $\mathbb{I}[\cdot]$ is the indicator function taking value one if the proposition is true, zero otherwise.

Therefore, we can apply the same decomposition to all tasks~$\set{Y}$ and rearrange terms as follows.
\begin{align}
\mathcal{L}(\set{\theta} \mid \vec{x},\vec{c},\vec{y},\vec{h})
\\
=
&p(\vec{x})
p(h_Y \mid \vec{x}, \vec{c})
\bigl(
p(M_Y = y \mid \vec{c}, \theta_Y)
+
\mathbb{I}[h_Y = y]
p(M_Y = \bot \mid \vec{c}, \theta_Y)
\bigr)
p(\vec{c},\vec{y}_{\setminus Y},\vec{h}_{\setminus Y} \mid \vec{x}, \set{\theta}_{\setminus \theta_Y})\\
=
&p(\vec{x})
p(h_Y \mid \vec{x}, \vec{c})
\prod_{Y\in\set{Y}}
\bigl(
p(M_Y = y \mid \vec{c}, \theta_Y)
+
\mathbb{I}[h_Y = y]
p(M_Y = \bot \mid \vec{c}, \theta_Y)
\bigr)
p(\vec{c},\vec{h}_{\set{C}} \mid \vec{x}, \set{\theta}_{\set{C}})\\
=
&p(\vec{x})
p(h_Y \mid \vec{x}, \vec{c})
p(\vec{c},\vec{h}_{\set{C}} \mid \vec{x}, \set{\theta}_{\set{C}})
\prod_{Y\in\set{Y}}
\bigl(
p(M_Y = y \mid \vec{c}, \theta_Y)
+
\mathbb{I}[h_Y = y]
p(M_Y = \bot \mid \vec{c}, \theta_Y)
\bigr).
\end{align}

Then, we can apply a similar decomposition to concepts, starting for an arbitrary concept $C\in\set{C}$.
\begin{align}
    &p(\vec{c},\vec{h}_{\set{C}} \mid \vec{x},\theta_{\set{C}})\\
    =
    &\sum_{k\in[n_C+1]}
    p(\vec{c},\vec{h}_{\set{C}},M_C=c \mid \vec{x},\theta_{\set{C}})\\
    =
    &
    p(h_C \mid\vec{x})
    \sum_{k\in[n_C+1]}
    p(C=c\mid M_C=k, h_C)P(M_C=c\mid \vec{x},\theta_C) \cdot p(\vec{c}_{\setminus C}, h_{\set{C}\setminus C}\mid \vec{x}, \theta_{\set{C}\setminus C})\\
    =
    &
    p(h_C \mid\vec{x})
    \bigl(
    p(M_C = c \mid \vec{x}, \theta_C)
    +
    \mathbb{I}[h_C = c]
    p(M_C = \bot \mid \vec{x}, \theta_C).
    \bigr)
    \cdot p(\vec{c}_{\setminus C}, h_{\set{C}\setminus C}\mid \vec{x}, \theta_{\set{C}\setminus C})\\
    =
    &
    p(h_{\set{C}} \mid\vec{x})
    \prod_{C\in\set{C}}
    \bigl(
    p(M_C = c \mid \vec{x}, \theta_C)
    +
    \mathbb{I}[h_C = c]
    p(M_C = \bot \mid \vec{x}, \theta_C)
    \bigr).
\end{align}

Finally, leading to the following form, which we further simplify by denoting as $\vec{z}_V$ the input of each variable $V\in\set{V}$ and as $v\in\dom{V}$ its realization in the dataset.
\begin{align}
\mathcal{L}(\set{\theta} \mid \vec{x},\vec{c},\vec{y},\vec{h})\\
\begin{split}
=p(\vec{x})
p(\vec{h}_{\set{C}}\mid\vec{x})
p(\vec{h}_{\set{Y}}\mid\vec{x},\vec{c})
\prod_{C\in\set{C}}
\left(
p(M_C = c \mid \vec{x}, \theta_C)
+
\mathbb{I}[h_C = c]
p(M_C = \bot \mid \vec{x}, \theta_C)
\right)\\
\cdot\prod_{Y\in\set{Y}}
\left(
p(M_Y = v \mid \vec{c}, \theta_Y)
+
\mathbb{I}[h_Y = y]
p(M_Y = \bot \mid \vec{c}, \theta_Y).
\right)
\end{split}
\\
&=
p(\vec{x})
p(\vec{h}\mid\vec{x},\vec{c})
\prod_{V\in\set{V}}
\left(
p(M_V = v \mid \vec{z}_V, \theta)
+
\mathbb{I}[h_V = v]
p(M_V = \bot \mid \vec{z}_V, \theta_V).
\right)
\end{align}

Finally,
we show that maximizing the likelihood
equates to minimizing the loss function
we defined in \Cref{sec:method}.
We recall that to this end, we assume to have for each variable $V\in\set{V}$ a machine learning model $g(\cdot; \theta_V)$ that produces $n_V+1$ activations, one for each class and one additional for the defer action.
\begin{align}
\hat{\set{\theta}}
&=
\argmax_{\set{\theta}}
\mathcal{L}(\set{\theta} \mid \vec{x},\vec{c},\vec{y},\vec{h})
\\&=
\argmax_{\set{\theta}}
p(\vec{x})p(\vec{h}\mid\vec{x},\vec{c})
\prod_{V\in\set{V}}
p(M_V = v \mid \vec{z}_V, \theta)
+
\mathbb{I}[h_V = v]
p(M_V = \bot \mid \vec{z}_V, \theta_V)
\\&=
\argmax_{\set{\theta}}
\prod_{V\in\set{V}}
p(M_V = v \mid \vec{z}_V, \theta)
+
\mathbb{I}[h_V = v]
p(M_V = \bot \mid \vec{z}_V, \theta_V)
\\&=
\argmax_{\set{\theta}}
\sum_{V\in\set{V}}
\log(
p(M_V = v \mid \vec{z}_V, \theta)
+
\mathbb{I}[h_V = v]
p(M_V = \bot \mid \vec{z}_V, \theta_V))
\label{eq:preineq}
\\&=
\argmax_{\set{\theta}}
\sum_{V\in\set{V}}
\log(
p(M_V = v \mid \vec{z}_V, \theta))
+
\mathbb{I}[h_V = v]
\log(
p(M_V = \bot \mid \vec{z}_V, \theta_V))
\label{eq:postineq}
\\&=
\argmin_{\set{\theta}}
\sum_{V\in\set{V}}
-\log(
p(M_V = v \mid \vec{z}_V, \theta))
-
\mathbb{I}[h_V = v]
\log(
p(M_V = \bot \mid \vec{z}_V, \theta_V))
\\&=
\argmin_{\set{\theta}}
\sum_{V\in\set{V}}
\Psi(q(\vec{z}_V; \theta_V), v)
+
\mathbb{I}[h_V = v]
\Psi(q(\vec{z}_V; \theta_V), \bot),
\end{align}
where $\Psi(g(\vec{z}_V; \theta_V))$ then corresponds to the standard formulation with the softmax operator, reported in \Cref{tab:liu}. In the same table, we report alternative formulations for the same object from the learning to defer literature. Further, we can justify the transition from \Cref{eq:preineq} to \Cref{eq:postineq} since the following holds
\begin{equation}
\begin{split}
&\log(
p(M_V = v \mid \vec{z}_V, \theta)
+
\mathbb{I}[h_V = v]
p(M_V = \bot \mid \vec{z}_V, \theta_V))\\
\geq
&\log(
p(M_V = v \mid \vec{z}_V, \theta))
+
\mathbb{I}[h_V = v]
\log(
p(M_V = \bot \mid \vec{z}_V, \theta_V)).
\end{split}
\end{equation}

\newpage
\subsection{Regularized Optimization of DCBM}%
\label{app:constrained}

As discussed in \Cref{subsec:dcbm}, we regularize the model to avoid trivially deferring whenever the human is correct.
In this way, we can account for the cost of deferring and relegating it to the most significative cases.
We formalize this intuition by requiring the log-probability of deferring when the human is correct to be smaller then zero.
Formally, we define the following constraint over all variables of the deferring system
\begin{align}
&\forall V\in\set{V}.\quad
\mathbb{E}_{\vec{c,h,x,y}}\left[
\mathbb{I}[h_V=v]
\log P(M_V = \bot \mid \vec{x}, \theta_V)
\right] < 0\\
\iff&\forall V\in\set{V}.\quad
\mathbb{E}_{\vec{c,h,x,y}}\left[
-1 \cdot
\mathbb{I}[h_V=v]
\log P(M_V = \bot \mid \vec{x}, \theta_V)
\right] > 0\\
\iff&\forall V\in\set{V}.\quad
\mathbb{E}_{\vec{c,h,x,y}}\left[
\mathbb{I}[h_V=v]
\Psi(q(\vec{z}_V; \theta_V), \bot)
\right] > 0.
\end{align}

In practice, we treat the constraint as a regularization term controlled by an hyperparameter $\lambda\in\real$. In particular, let
\begin{equation}
g_V(\vec{c},\vec{h},\vec{x}, \vec{y}) =
\mathbb{I}[h_V=v]
\Psi(q(\vec{z}_V; \theta_V), \bot),
\end{equation}
be the value of the constraint on the variable $V\in\set{V}$.
We treat the constrained optimization problem
as the following regularized unconstrained problem.
\begin{align}
&\min_{\set{\theta}}
\mathbb{E}_{\vec{c,h,x,y}}\left[
\ell(\set{\theta}\mid\vec{c,h,x,y})
\right]
- \lambda
\sum_{V\in\set{V}}
\mathbb{E}_{\vec{c,h,x,y}}\left[
g_V(\vec{c},\vec{h},\vec{x}, \vec{y})
\right]\\
=&\min_{\set{\theta}}
\mathbb{E}_{\vec{c,h,x,y}}\left[
\ell(\set{\theta}\mid\vec{c,h,x,y})
\right]
+
\mathbb{E}_{\vec{c,h,x,y}}\left[
- \lambda
\sum_{V\in\set{V}}
g_V(\vec{c},\vec{h},\vec{x}, \vec{y})
\right]\\
=&\min_{\set{\theta}}
\mathbb{E}_{\vec{c,h,x,y}}\left[
\ell(\set{\theta}\mid\vec{c,h,x,y})
- \lambda
\sum_{V\in\set{V}}
g_V(\vec{c},\vec{h},\vec{x}, \vec{y})
\right]\\
=&\min_{\set{\theta}}
\mathbb{E}_{\vec{c,h,x,y}}\left[
\sum_{V\in\set{V}}
\Psi(q(\vec{z}_V; \theta_V), v)
+
\mathbb{I}[h_V = v]
\Psi(q(\vec{z}_V; \theta_V), \bot)
-\lambda
\mathbb{I}[h_V = v]
\Psi(q(\vec{z}_V; \theta_V), \bot)
\right]\\
=&\min_{\set{\theta}}
\mathbb{E}_{\vec{c,h,x,y}}\left[
\sum_{V\in\set{V}}
\Psi(q(\vec{z}_V; \theta_V), v)
+ (1 - \lambda)
\mathbb{I}[h_V = v]
\Psi(q(\vec{z}_V; \theta_V), \bot)
\right].
\end{align}

Further, we show that the formulation from \citet{DBLP:conf/aistats/LiuCZF024} arises when explicitly constraining the model to avoid deferring whenever the human is incorrect in the training distribution as in ${\mathbb{E}\left[\mathbb{I}[h_V\neq V]P(M_V\neq\bot\mit\vec{x},\theta_V)\right]>0}$. By expressing the constraint in terms of log-probabilities, we get the following result.
\begin{align}
&\forall V\in\set{V}.\quad
\mathbb{E}_{\vec{c,h,x,y}}\left[
\mathbb{I}[h_V \neq v]
\log P(M_V \neq \bot \mid \vec{x}, \theta_V)
\right] > -\epsilon,\\
\iff&\forall V\in\set{V}.\quad
\mathbb{E}_{\vec{c,h,x,y}}\left[
\mathbb{I}[h_V \neq v]
\log
\sum_{k\in [n_V]}
P(M_V =k \mid \vec{x}, \theta_V)
\right] > -\epsilon,\\
\impliedby&\forall V\in\set{V}.\quad
\mathbb{E}_{\vec{c,h,x,y}}\left[
\mathbb{I}[h_V \neq v]
\sum_{k\in [n_V]}
\log P(M_V =k \mid \vec{x}, \theta_V)
\right] > -\epsilon,\\
\iff&\forall V\in\set{V}.\quad
\mathbb{E}_{\vec{c,h,x,y}}\left[
-1 \cdot
\mathbb{I}[h_V \neq v]
\sum_{k\in [n_V]}
\log P(M_V =k \mid \vec{x}, \theta_V)
\right] < \epsilon,\\
\iff&\forall V\in\set{V}.\quad
\mathbb{E}_{\vec{c,h,x,y}}\left[
\mathbb{I}[h_V \neq v]
\sum_{k\in [n_V]}
\Psi(q(\vec{z}_V; \theta_V), v)
\right] < \epsilon,
\end{align}
for a positive threshold $\epsilon>0$.
Consequently, when introducing this constraints with the same penalty $\lambda$ in the optimization problem, we obtain the following formulation
\begin{equation}
\min_{\set{\theta}}
\mathbb{E}_{\vec{c,h,x,y}}\left[
\sum_{V\in\set{V}}
\Psi(q(\vec{z}_V; \theta_V), v)
+ (1 - \lambda)
\mathbb{I}[h_V = v]
\Psi(q(\vec{z}_V; \theta_V), \bot)
+ \lambda
\mathbb{I}[h_V \neq v]
\sum_{k\in[n_V]}
\Psi(q(\vec{z}_V; \theta_V), k)
\right].
\end{equation}

\begin{table*}
\begin{center}
\resizebox{\textwidth}{!}{
\begin{tabular}{c|c}
    Loss Name & Loss Function \\
    \toprule
    CE~\citep{DBLP:conf/aistats/MozannarLWSDS23} & $\psi\left(q(z), k\right) = -\log\left(\frac{\exp\left(q(z)_k\right)}{\sum_{k'\in [K+1]}\exp\left(q(z)_{k'}\right)}\right)$ \\
    \midrule
    \\
    OVA~\citep{DBLP:conf/aistats/VermaBN23} & $\psi\left(q(z), k\right)=\begin{cases}
        \log\left(1+\exp\left( -q(z)_k\right)\right)-\log\left(1+\exp\left( +q(z)_k\right)\right)
        &\text{if } k=\perp \\
        \log\left(1+\exp\left( -q(z)_k\right)\right)+\sum_{k'\in [K+1]/\{k\}}\log\left(1+\exp\left( +q(z)_{k'}\right)\right)
        &\text{otherwise}
    \end{cases}$\\
    \midrule
    \\
    ASM~\citep{DBLP:conf/nips/CaoM0W023} & $\psi\left(q(z), k\right) = \begin{cases}
     -\log\left(\frac{\exp\left(q(z)_k\right)}{\sum_{k'\in [K]}\exp\left(q(z)_{k'}\right) - \max_{k'\in [K]}\exp\left(q(z)_{k'}\right) }\right) &\text{if } k=\perp\\
     -\log\left(\frac{\exp\left(q(z)_k\right)}{\sum_{k'\in [K]}(\exp\left(q(z)_{k'}\right)}\right)-\log\left(\frac{\sum_{k'\in [K]}\exp\left(q(z)_{k'}\right) - \max_{k'\in [K]}\exp\left(q(z)_{k'}\right)}{\sum_{k'\in [K+1]}\exp\left(q(z)_{k'}\right) - \max_{k'\in [K]}\exp\left(q(z)_{k'}\right)}\right) &\text{otherwise}
     \end{cases}$ \\
\end{tabular}
}
\end{center}
\caption{Multiclass losses from \citet{DBLP:conf/aistats/LiuCZF024}.}%
\label{tab:liu}
\end{table*}

\subsection{Lemma 3.2 --- Sum of Consistent Losses}%
\label{app:proofsum}

Before proving \cref{lemma:sumconsistent}, we prove the following result that is a fundamental property arising from the definition of the argmin function and the independence of variables.

\begin{lemma}
\label{lem:argmin}
    Given $f,g:A\subseteq\mathbb{R}^n\to \mathbb{R}$, we have
    \[
\argmin_{(x,y)\in A^2} \left(f(x)+g(y)\right)=\{(\bar{x},\bar{y})\in A^2:\ \bar{x}\in\argmin_{x\in A} f(x), \bar{y}\in \argmin_{y\in A} g(y)\}
    \]
    \begin{proof}
    For the sake of simplicity we use the following shortcut:
    \[
    L = \argmin_{(x,y)\in A^2} \left(f(x)+g(y)\right) \quad\mbox{and}\quad R = \{(\bar{x},\bar{y})\in A^2:\ \bar{x}\in\argmin_{x\in A} f(x)\wedge \bar{y}\in \argmin_{y\in A} g(y)\}
    \]
First, we notice that in case any between $f$ or $g$ has no minimum in $A$, then the claim is trivially proved as $L=R=\emptyset$. Indeed, let us assume, e.g., that $f$ has no minimum in $A$, then clearly $R=\emptyset$. Moreover, also $L=\emptyset$. Indeed, if we assume by contradiction that $L\neq\emptyset$, then it exists $(\bar{x},\bar{y})\in L$, i.e.  $(\bar{x},\bar{y})\in A^2$ with $f(\bar{x})+g(\bar{y})\leq f(x)+g(y)$ for every $(x,y)\in A^2$. By taking $y=\bar{y}$ and canceling $g(\bar{y})$ on both sides we get that $f(\bar{x})\leq f(x)$ for every $x\in A$. Therefore  $f$ has at least a minimum ($\bar{x}$) in $A$, which is a contradiction, so it must be $L=\emptyset$, as well.

So lets consider the case of both $L\neq\emptyset$ and $R\neq\emptyset$. We show the double inclusion.

    \begin{enumerate}
      \item If $(\bar{x},\bar{y})\in L$ then $f(\bar{x})+g(\bar{y})\leq f(x)+g(y)$ for every $(x,y)\in A^2$. From this inequality, by taking $x=\bar{x}$ and canceling $f(\bar{x})$, we get $\bar{y}\in\argmin_{y\in A}g(y)$. Identically, by taking $y=\bar{y}$, we get $\bar{x}\in\argmin_{x\in A}f(x)$. Therefore $(\bar{x},\bar{y})\in R$.
            \item If $(\bar{x},\bar{y})\in R$ then $\bar{x}\in\argmin_{x\in A}f(x)$ and $\bar{y}\in\argmin_{y\in A}g(y)$. Namely, $f(\bar{x})\leq f(x)$ for every $x\in A$ and $g(\bar{y})\leq g(y)$ for every $y\in A$. By summing on both sides, we get $f(\bar{x})+g(\bar{y})\leq f(x)+g(y)$ for every $(x,y)\in A^2$, and so $(\bar{x},\bar{y})\in L$
    \end{enumerate}
    \end{proof}
\end{lemma}

\textbf{\cref{lemma:sumconsistent}}
Let $\ell^\prime_1, \ell_1, \cdots,\ell^\prime_m, \ell_m$  be (possibly distinct) loss functions. Assume that, for every $i\in\{1,\ldots,m\}$, $\ell^\prime_i,\ell_i:\mathbb{R}^{n_i}\to\mathbb{R}$, being $\ell'_i$ a consistent surrogate of $\ell_i$.
Then $\ell^\prime:\mathbb{R}^{n}\to\mathbb{R}$, with $n=n_1+\ldots+n_m$ and 
$\ell'(\theta_1,\ldots,\theta_m)=\sum_{i=1}^m\ell^\prime_i(\theta_i)$
     is a consistent surrogate of
$\ell:\mathbb{R}^{n}\to\mathbb{R}$, with $\ell(\theta_1,\ldots,\theta_m)=\sum_{i=1}^m\ell(\theta_i)$.
 \begin{proof}
 The proof is a direct consequence of \Cref{lem:argmin}. For simplicity, we show the complete proof for $m=2$.
 To be precise, from the statement we report explicitly that, $\ell'_1,\ell_1:\mathbb{R}^{n_1}\to \mathbb{R}$, $\ell'_2,\ell_2:\mathbb{R}^{n_2}\to \mathbb{R}$  and $\ell',\ell:\mathbb{R}^{n}\to \mathbb{R}$ with $n=n_1+n_2$,  $\ell'=\ell'_1+\ell'_2$ and $\ell=\ell_1+\ell_2$. We have to prove that $\ell'$ is a consistent surrogate of $\ell$, namely that $\argmin_{\theta\in\mathbb{R}^n} \ell'(\theta)\subseteq \argmin_{\theta\in\mathbb{R}^n} \ell(\theta)$.

 Let $\theta^*=(\theta^*_1,\theta^*_2)\in\mathbb{R}^{n_1+n_2}$ be a minimum of $\ell'$ (the claim would be trivial in case  $\ell'$ has no minima). Then according to \Cref{lem:argmin}, we have:
 \[
 \theta^* \in\argmin_{\theta\in\mathbb{R}^n}\ell'(\theta)=\argmin_{(\theta_1,\theta_2)\in\mathbb{R}^{n_1+n_2}}\left(\ell'_1(\theta_1)+\ell'_2(\theta_2)\right) = \{(\bar{\theta}_1,\bar{\theta}_2)\in\mathbb{R}^{n_1+n_2}:\bar{\theta}_1\in\argmin_{\theta_1\in\mathbb{R}^{n_1}}\ell'_1(\theta_1) \wedge \bar{\theta}_2\in\argmin_{\theta_2\in\mathbb{R}^{n_2}}\ell'_2(\theta_2)\}
 \]
Therefore $\theta_1^*\in\argmin_{\theta_1\in\mathbb{R}^{n_1}}\ell'_1(\theta_1)$ and $\theta_2^*\in\argmin_{\theta_2\in\mathbb{R}^{n_2}}\ell'_2(\theta_2)$.
Since by hypothesis $\ell'_1,\ell'_2$ are consistent surrogates of $\ell_1,\ell_2$, respectively, it follows that:
$\theta_1^*\in\argmin_{\theta_1\in\mathbb{R}^{n_1}}\ell_1(\theta_1)$ and $\theta_2^*\in\argmin_{\theta_2\in\mathbb{R}^{n_2}}\ell_2(\theta_2)$.

Finally, the proof concludes by using again \Cref{lem:argmin}:
\[
 \theta^*\in\{(\bar{\theta}_1,\bar{\theta}_2)\in\mathbb{R}^{n_1+n_2}:\bar{\theta}_1\in\argmin_{\theta_1\in\mathbb{R}^{n_1}}\ell_1(\theta_1) \wedge \bar{\theta}_2\in\argmin_{\theta_2\in\mathbb{R}^{n_2}}\ell_2(\theta_2)\}=\argmin_{(\theta_1,\theta_2)\in\mathbb{R}^{n_1+n_2}}\left(\ell_1(\theta_1)+\ell_2(\theta_2)\right)= \argmin_{\theta\in\mathbb{R}^n}\ell(\theta)
 \]

 \end{proof}

\subsection{Theorem 1 --- Sum of Consistent Losses}%
\label{app:proofmain}

\begin{proof}
    By the previous \cref{lemma:sumconsistent},
    the sum of consistent losses is consistent to the sum of the target loss functions.
    It is thus immediate how this applies our optimization problem both for the unconstrained (Equation \ref{eq:psiloss}) and the penalized (Equation \ref{eq:psilosspenalized}) losses.
    Formally,
    \begin{align}
    \sum_{V\in\set{V}}
    \Psi(q(\vec{z}_V; \theta_V), v)
    + (1 - \lambda)
    \mathbb{I}[h_V = v]
    \Psi(q(\vec{z}_V; \theta_V), \bot)
    \end{align}
    is the sum of losses consistent of the zero-one loss which we reported in Equation \ref{eq:L2D_min} whenever $\lambda=1$. In fact, it corresponds to an equivalent formulation in Theorem 1 from \citet{DBLP:conf/icml/MozannarS20}.
    Similarly, for any other $\lambda\in[0,1]$, the penalized version coincides in the single-variable case to the provably consistent formulation from Equation 4 in \citet{DBLP:conf/aistats/LiuCZF024}. Therefore, the sum over different variables is consistent to the sum of the zero-one loss.
\end{proof}

\section{Experimental Details}%
\label{app:exp_det}

\paragraph{Data Split.}
For the \texttt{completeness} synthetic dataset, we sample $1,000$ instances with an $80\%$-$20\%$ train-test split ratio. For \texttt{cifar10h}, we randomly split the dataset into training, validation and test according to a $70\%, 10\%, 20\%$ ratio. For \texttt{CUB}, we keep the original split.

\paragraph{Concept and Task Predictors}
For each concept predictor $q_C$, we employ a three-layer MLP with a leaky-relu activation function. For the \texttt{completeness} dataset, each concept encoder model takes as input the raw data. For the image datasets \texttt{cifar10-h} and \texttt{cub}, concept predictors take instead as an input the pre-trained embedding discussed in \Cref{subsec:pretrain}. For \texttt{CUB}, we obtain such an embedding by training a ResNet34~\citep{DBLP:conf/cvpr/HeZRS16} for 100 epochs to solve the final task using a cross-entropy loss function. The representations obtained by the pre-trained model are then frozen and used as the input for each concept encoder. For \texttt{cifar10-h} we consider the pre-trained WideResNet~\citep{DBLP:conf/bmvc/ZagoruykoK16} provided by \citet{DBLP:journals/corr/abs-2405-18902}, who trained a WideResNet architecture on the original \texttt{cifar10}~\citep{cifar10} training set for 200 epochs. We use the obtained representations to train all the concept encoders. All the final task classifiers consist of another three-layer MLP taking as input the concept values.

\paragraph{Training Procedure.}
We train every combination of models and defer costs~$\lambda$ for 100 epochs. For \texttt{completeness}, we use \texttt{Adam}~\citep{DBLP:journals/corr/KingmaB14} with a learning rate equal to $.001$ and no scheduler. For both \texttt{cifar10-h} and \texttt{CUB}, we use \texttt{AdamW}~\citep{DBLP:conf/iclr/LoshchilovH19} as an optimizer, setting the initial learning rate to $.001$. We decrease the learning rate every 25 epochs by $.5$. Additionally, for \texttt{CUB}, following \citet{DBLP:conf/nips/ZarlengaBCMGDSP22} guidelines, we consider a weighted version of the loss on concepts to take into account their imbalance. To limit the computational burden, for both \texttt{cifar10-h} and \texttt{CUB}, we perform early stopping after 10 epochs if there is no improvement for the loss on the validation set.

\paragraph{Evaluation.}
All the results are averaged over five and three runs on the synthetic and the other datasets, respectively, with fixed datasets' splits.



\begin{figure*}
    \centering
    \includegraphics[width=0.9\textwidth]{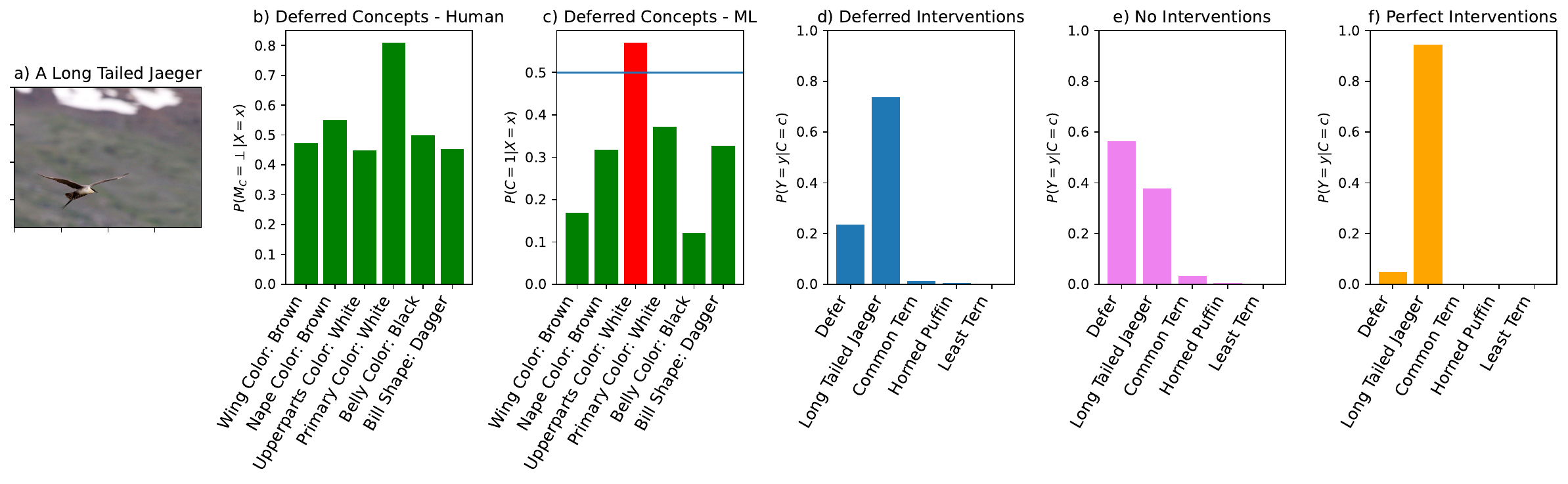}
    \includegraphics[width=0.9\textwidth]{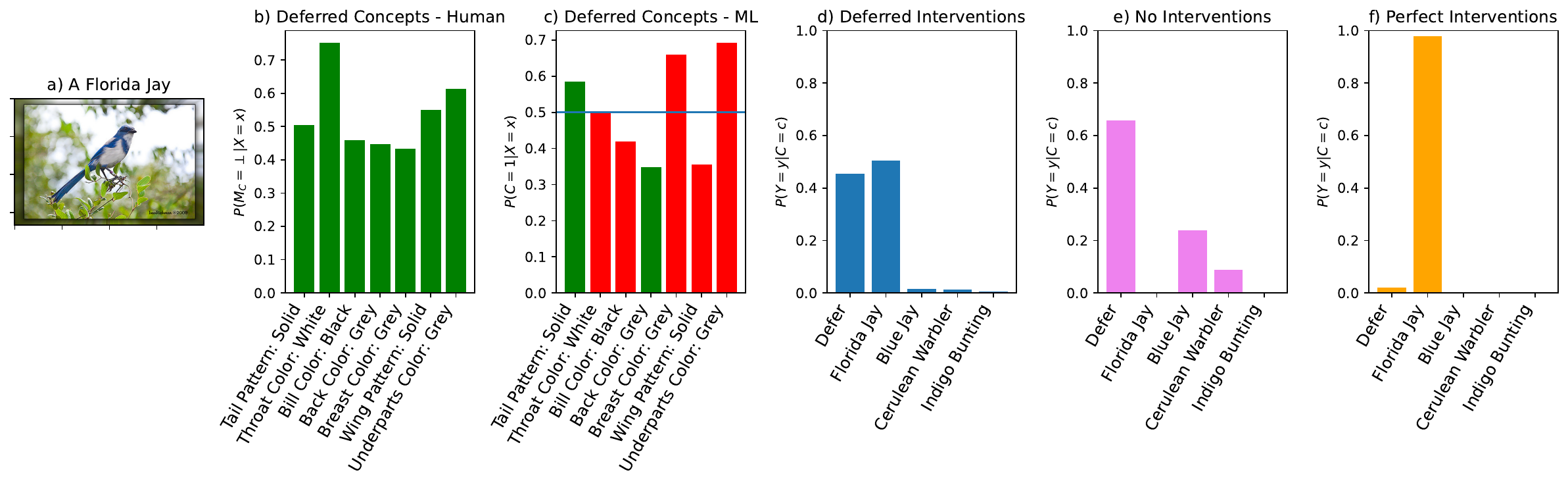}
    \includegraphics[width=0.9\textwidth]{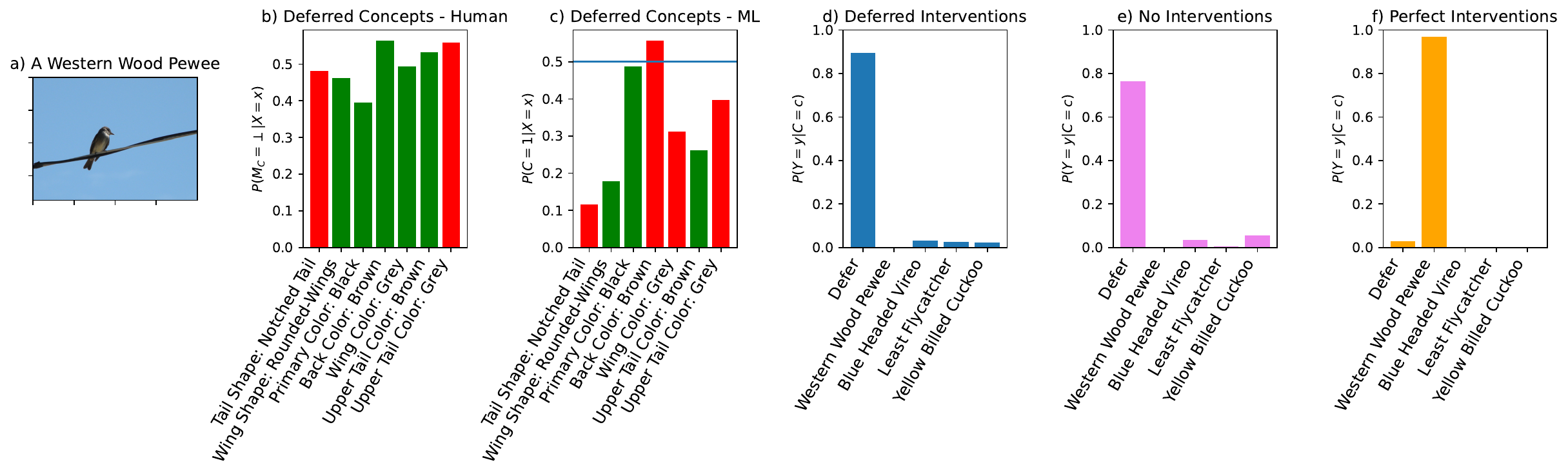}
    \includegraphics[width=0.9\textwidth]{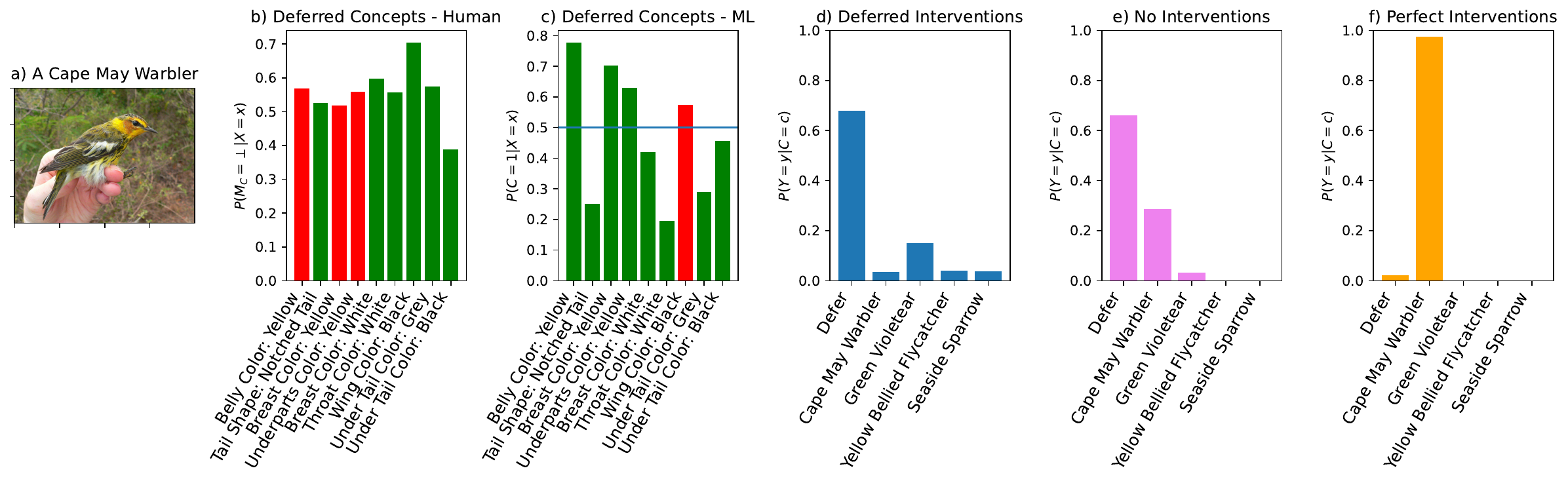}
    \caption{Interpretation of a \DCBM{} with defer cost $\lambda=0.1$ on an input sample.
    From left to right:
    (a.) examples of an image from the \texttt{CUB} dataset;
    (b.) the concepts that the model has deferred with the estimated probability, green bars stand for when the human correctly predicts the concept, red otherwise;
    (c.) the estimated probability of each deferred concept being true according to the machine learning model, green bars stand for when the ML would have correctly predicted the concept, red otherwise;
    (d.) the estimated probability of top-5 final task labels after deferring the concepts to the human (standard \DCBM{} behavior);
    (e.) the estimated probability of top-5 final task labels without deferring the concepts to the human;
    (f.) the estimated probability of top-5 final task labels from the ground-truth concepts;
    }%
    \label{fig:explanationsapp}
\end{figure*}

\section{Additional Explanations}%
\label{app:interpretations}

In \Cref{fig:explanationsapp}, we report additional samples from \texttt{CUB} that we analyze as in the experimental results for research question Q4 in \Cref{subsec:expres}. 
While in the first two rows the effect of deferring on concepts is clearly beneficial, the other two examples require further discussion.

In both cases, we can see that the human would still mispredict a few concepts. Interestingly, such concepts are not visible in the image (e.g., the upper tail in the second-to-last example and the belly colour in the last example).  Therefore, the perfect interventions are to be considered an ideal scenario only, as in practice, these concepts are not directly predictable from the image. Still, \DCBM{} can ``correctly defer'' on the final task as concept interventions would not suffice to disambiguate the correct label.

\section{Additional Results}%
\label{app:additionalresults}

We provide additional results on the synthetic
dataset \texttt{completeness} to investigate the following comparisons:
\begin{itemize}
    \item Shared parameters among the concept encoders,
    \item Different human expert accuracy on both the concepts and task,
    \item Joint vs Independent training,
    \item Different learning-to-defer losses.
\end{itemize}

\paragraph{Label Smoothing.}
\citet{DBLP:conf/aistats/LiuCZF024} studies the problem of label smoothing on learning to defer losses a proposes a slightly different formulation of Equation~\ref{eq:psilosspenalized}, which, once adapted to our notation, we report as follows:
\begin{equation}
\label{eq:psilosspenalized_ls}
\sum_{V \in\set{V}}
\Psi(q(\vec{z}_V;\theta_V), v)\\
+
(1-\lambda) \cdot \mathbb{I}\left[y_V = h_V\right]
\Psi(q(\vec{z}_V; \theta_V), \bot)\\
+
\lambda \cdot \mathbb{I}\left[y_V \neq h_V\right]
\argmin_{k\in[K]}
\Psi(q(\vec{z}_V; \theta_V), k),
\end{equation}
In all the coming experimental results, we use the suffix \texttt{-LS} to refer to the results using Equation~\ref{eq:psilosspenalized_ls}, while we use the suffix \texttt{-NLS} for the formulation in Equation~\ref{eq:psilosspenalized}. As it is shown in the upcoming additional results, we do not observe noteworthy differences in the performance of the two loss functions.

\paragraph{Joint Learning.}
While independent training is required to guarantee the consistency of the learning-to-defer loss function, we implement joint learning to compare empirically.
We implement the joint learning strategy by considering the following soft-labeled concept predictor:
\begin{equation}
\tilde{g}(\vec{x}) =
g_1(\vec{x}) (1-g_\perp(\rvx))
+
h_c (g_\perp(\rvx)),
\end{equation}
where we produce the output as the weighted sum of the human-provided concept~$\psi$ and the machine learning model concept. The weight corresponds to the probability of deferring or not the instance.

First, we study the different negative log-likelihood terms (\Cref{tab:liu}) that can be employed within the learning-to-defer (Equation~\ref{eq:psilosspenalized}) loss function both for both independent~(\Cref{%
tab:frozen-CEoracle-TEoracle-independent-ASM-results,%
tab:frozen-CEoracle-TEoracle-independent-CE-results,%
tab:frozen-CEoracle-TEoracle-independent-OVA-results%
})
and joint learning~(\Cref{%
tab:frozen-CEoracle-TEoracle-joint-ASM-results,%
tab:frozen-CEoracle-TEoracle-joint-CE-results,%
tab:frozen-CEoracle-TEoracle-joint-OVA-results%
}),
when dealing with oracle human experts on both concepts and tasks.
Further, for the ASM loss function, we also study how the model behaves when we do not freeze the parameters of the encoder, also for independent%
~(\Cref{tab:unfrozen-CEoracle-TEoracle-independent-ASM-results})
and joint training%
~(\Cref{tab:unfrozen-CEoracle-TEoracle-joint-ASM-results}).
Finally, we study multiple combinations of human expertize on the concepts and the tasks, whose reference we summarize in the following table:

\begin{center}
\begin{tabular}{l|lccc}
& & \multicolumn{3}{c}{Human Task Expert}\\
& 60 & 80 & oracle\\
\midrule
\multirow{3}{*}{Human Concept Expert}
& 60
&\Cref{%
tab:frozen-CEhuman60-TEhuman60-independent-ASM-results,%
tab:frozen-CEhuman60-TEhuman60-joint-ASM-results,%
}
&\Cref{%
tab:frozen-CEhuman60-TEhuman80-independent-ASM-results,%
tab:frozen-CEhuman60-TEhuman80-joint-ASM-results,%
}
&\Cref{%
tab:frozen-CEhuman60-TEoracle-independent-ASM-results,%
tab:frozen-CEhuman60-TEoracle-joint-ASM-results,%
}\\
& 80
&\Cref{%
tab:frozen-CEhuman80-TEhuman60-independent-ASM-results,%
tab:frozen-CEhuman80-TEhuman60-joint-ASM-results,%
}
&\Cref{%
tab:frozen-CEhuman80-TEhuman80-independent-ASM-results,%
tab:frozen-CEhuman80-TEhuman80-joint-ASM-results,%
}
&\Cref{%
tab:frozen-CEhuman80-TEoracle-independent-ASM-results,%
tab:frozen-CEhuman80-TEoracle-joint-ASM-results,%
}\\
& oracle
&\Cref{%
tab:frozen-CEoracle-TEhuman60-independent-ASM-results,%
tab:frozen-CEoracle-TEhuman60-joint-ASM-results,%
}
&\Cref{%
tab:frozen-CEoracle-TEhuman80-independent-ASM-results,%
tab:frozen-CEoracle-TEhuman80-joint-ASM-results,%
}
&\Cref{%
tab:frozen-CEoracle-TEoracle-independent-ASM-results,%
tab:frozen-CEoracle-TEoracle-joint-ASM-results,%
}
\end{tabular}
\end{center}

\begin{table}
\centering
\caption{Results for the \texttt{completeness} dataset when not allowing for shared parameters with independent training using ASM, and considering oracle task expert and oracle concept expert. LS refers to the label-smoothing-free implementation, while NLS to the one with label smoothing. We report $avg\pm std$ and highlight the best baseline in bold.}
\label{tab:frozen-CEoracle-TEoracle-independent-ASM-results}
\resizebox{\textwidth}{!}{\begin{tabular}{c|c|cccccc}
\toprule
                                                       Metric &  $\lambda$ &         \texttt{DCBM-LS} &       \texttt{DCBM-NC-LS} &       \texttt{DCBM-NT-LS} &        \texttt{DCBM-NLS} &      \texttt{DCBM-NC-NLS} &      \texttt{DCBM-NT-NLS} \\
\midrule
\midrule\multirow{12}{*}{\rotatebox[origin=c]{90}{$AccTask$}} &       0.00 & $\mathbf{1.000\pm.000}$ & $\mathbf{1.000\pm.000}$ &           $.925\pm.010$ & $\mathbf{1.000\pm.000}$ & $\mathbf{1.000\pm.000}$ &           $.909\pm.030$ \\
                                                              &       0.01 & $\mathbf{1.000\pm.000}$ &           $.998\pm.004$ &           $.919\pm.022$ & $\mathbf{1.000\pm.000}$ & $\mathbf{1.000\pm.000}$ &           $.912\pm.019$ \\
                                                              &       0.05 &           $.997\pm.004$ &           $.987\pm.012$ &           $.906\pm.025$ &  $\mathbf{.998\pm.003}$ &           $.982\pm.012$ &           $.907\pm.026$ \\
                                                              &       0.10 &  $\mathbf{.982\pm.007}$ &           $.922\pm.036$ &           $.904\pm.036$ &           $.980\pm.004$ &           $.880\pm.022$ &           $.919\pm.019$ \\
                                                              &       0.15 &           $.946\pm.012$ &           $.880\pm.009$ &           $.909\pm.007$ &  $\mathbf{.950\pm.018}$ &           $.884\pm.015$ &           $.907\pm.016$ \\
                                                              &       0.20 &           $.884\pm.009$ &           $.861\pm.008$ &  $\mathbf{.892\pm.009}$ &           $.886\pm.015$ &           $.844\pm.020$ &           $.885\pm.014$ \\
                                                              &       0.25 &           $.858\pm.008$ &           $.824\pm.014$ &           $.863\pm.022$ &  $\mathbf{.864\pm.020}$ &           $.838\pm.008$ &           $.850\pm.026$ \\
\bottomrule
\end{tabular}}
\end{table}
\begin{table}[t]
\centering
\caption{Results for the \texttt{completeness} dataset when not allowing for shared parameters with joint training using ASM, and considering oracle task expert and oracle concept expert. We report $avg\pm std$ and highlight the best baseline in bold.}
\label{tab:frozen-CEoracle-TEoracle-joint-ASM-results}
\resizebox{\textwidth}{!}{\begin{tabular}{c|c|cccccc}
\toprule
                                                       Metric &  $\lambda$ &         \texttt{DCBM-LS} &       \texttt{DCBM-NC-LS} &       \texttt{DCBM-NT-LS} &        \texttt{DCBM-NLS} &      \texttt{DCBM-NC-NLS} &      \texttt{DCBM-NT-NLS} \\
\midrule
\midrule\multirow{12}{*}{\rotatebox[origin=c]{90}{$AccTask$}} &       0.00 & $\mathbf{1.000\pm.000}$ & $\mathbf{1.000\pm.000}$ &           $.924\pm.022$ & $\mathbf{1.000\pm.000}$ & $\mathbf{1.000\pm.000}$ &           $.919\pm.017$ \\
                                                              &       0.01 & $\mathbf{1.000\pm.000}$ & $\mathbf{1.000\pm.000}$ &           $.912\pm.009$ & $\mathbf{1.000\pm.000}$ & $\mathbf{1.000\pm.000}$ &           $.932\pm.010$ \\
                                                              &       0.05 &           $.997\pm.004$ &           $.981\pm.021$ &           $.913\pm.025$ & $\mathbf{1.000\pm.000}$ &           $.980\pm.009$ &           $.898\pm.018$ \\
                                                              &       0.10 &  $\mathbf{.983\pm.008}$ &           $.883\pm.017$ &           $.924\pm.007$ &           $.982\pm.008$ &           $.909\pm.018$ &           $.917\pm.024$ \\
                                                              &       0.15 &  $\mathbf{.957\pm.026}$ &           $.867\pm.010$ &           $.921\pm.017$ &           $.940\pm.022$ &           $.865\pm.018$ &           $.910\pm.012$ \\
                                                              &       0.20 &  $\mathbf{.904\pm.023}$ &           $.856\pm.016$ &           $.881\pm.020$ &           $.902\pm.023$ &           $.845\pm.010$ &           $.887\pm.014$ \\
                                                              &       0.25 &           $.851\pm.026$ &           $.829\pm.010$ &           $.858\pm.018$ &           $.858\pm.022$ &           $.838\pm.014$ &  $\mathbf{.860\pm.008}$ \\
\bottomrule
\end{tabular}}
\end{table}
\begin{table}
\centering
\caption{Results for the \texttt{completeness} dataset when not allowing for shared parameters with independent training using CE, and considering oracle task expert and oracle concept expert. LS refers to the label-smoothing-free implementation, while NLS to the one with label smoothing. We report $avg\pm std$ and highlight the best baseline in bold.}
\label{tab:frozen-CEoracle-TEoracle-independent-CE-results}
\resizebox{\textwidth}{!}{\begin{tabular}{c|c|cccccc}
\toprule
                                                       Metric &  $\lambda$ &         \texttt{DCBM-LS} &        \texttt{DCBM-NC-LS} &        \texttt{DCBM-NT-LS} &        \texttt{DCBM-NLS} &       \texttt{DCBM-NC-NLS} &       \texttt{DCBM-NT-NLS} \\
\midrule
\midrule\multirow{12}{*}{\rotatebox[origin=c]{90}{$AccTask$}} &       0.00 & $\mathbf{.991\pm.004}$ &           $.939\pm.008$ &           $.912\pm.009$ &          $.988\pm.008$ &           $.943\pm.018$ &           $.895\pm.035$ \\
                                                              &       0.01 & $\mathbf{.994\pm.005}$ &           $.946\pm.023$ &           $.906\pm.018$ &          $.988\pm.004$ &           $.925\pm.022$ &           $.897\pm.019$ \\
                                                              &       0.05 &          $.981\pm.008$ &           $.927\pm.013$ &           $.894\pm.032$ & $\mathbf{.989\pm.007}$ &           $.919\pm.009$ &           $.896\pm.029$ \\
                                                              &       0.10 & $\mathbf{.981\pm.007}$ &           $.917\pm.015$ &           $.885\pm.029$ &          $.977\pm.008$ &           $.893\pm.014$ &           $.903\pm.019$ \\
                                                              &       0.15 & $\mathbf{.980\pm.009}$ &           $.898\pm.013$ &           $.895\pm.008$ & $\mathbf{.980\pm.009}$ &           $.910\pm.010$ &           $.902\pm.018$ \\
                                                              &       0.20 & $\mathbf{.958\pm.022}$ &           $.906\pm.017$ &           $.907\pm.011$ &          $.951\pm.028$ &           $.894\pm.018$ &           $.904\pm.007$ \\
                                                              &       0.25 & $\mathbf{.962\pm.003}$ &           $.889\pm.008$ &           $.905\pm.004$ &          $.958\pm.009$ &           $.876\pm.011$ &           $.887\pm.030$ \\
\bottomrule
\end{tabular}}
\end{table}
\begin{table}[t]
\centering
\caption{Results for the \texttt{completeness} dataset when not allowing for shared parameters with joint training using CE, and considering oracle task expert and oracle concept expert. We report $avg\pm std$ and highlight the best baseline in bold.}
\label{tab:frozen-CEoracle-TEoracle-joint-CE-results}
\resizebox{\textwidth}{!}{\begin{tabular}{c|c|cccccc}
\toprule
                                                       Metric &  $\lambda$ &         \texttt{DCBM-LS} &        \texttt{DCBM-NC-LS} &        \texttt{DCBM-NT-LS} &        \texttt{DCBM-NLS} &       \texttt{DCBM-NC-NLS} &       \texttt{DCBM-NT-NLS} \\
\midrule
\midrule\multirow{12}{*}{\rotatebox[origin=c]{90}{$AccTask$}} &       0.00 &          $.990\pm.008$ &           $.938\pm.016$ &           $.910\pm.019$ & $\mathbf{.993\pm.006}$ &           $.927\pm.028$ &           $.912\pm.019$ \\
                                                              &       0.01 & $\mathbf{.991\pm.007}$ &           $.930\pm.027$ &           $.901\pm.013$ &          $.987\pm.008$ &           $.931\pm.019$ &           $.917\pm.008$ \\
                                                              &       0.05 & $\mathbf{.987\pm.008}$ &           $.935\pm.019$ &           $.903\pm.020$ & $\mathbf{.987\pm.004}$ &           $.926\pm.012$ &           $.890\pm.021$ \\
                                                              &       0.10 &          $.976\pm.008$ &           $.911\pm.011$ &           $.917\pm.009$ & $\mathbf{.977\pm.008}$ &           $.917\pm.015$ &           $.906\pm.021$ \\
                                                              &       0.15 & $\mathbf{.978\pm.006}$ &           $.897\pm.010$ &           $.917\pm.010$ &          $.971\pm.013$ &           $.902\pm.017$ &           $.908\pm.021$ \\
                                                              &       0.20 &          $.958\pm.004$ &           $.902\pm.013$ &           $.895\pm.026$ & $\mathbf{.965\pm.008}$ &           $.899\pm.016$ &           $.904\pm.007$ \\
                                                              &       0.25 & $\mathbf{.971\pm.009}$ &           $.886\pm.015$ &           $.886\pm.027$ &          $.953\pm.014$ &           $.887\pm.009$ &           $.893\pm.020$ \\
\bottomrule
\end{tabular}}
\end{table}
\begin{table}
\centering
\caption{Results for the \texttt{completeness} dataset when not allowing for shared parameters with independent training using OVA, and considering oracle task expert and oracle concept expert. LS refers to the label-smoothing-free implementation, while NLS to the one with label smoothing. We report $avg\pm std$ and highlight the best baseline in bold.}
\label{tab:frozen-CEoracle-TEoracle-independent-OVA-results}
\resizebox{\textwidth}{!}{\begin{tabular}{c|c|cccccc}
\toprule
                                                       Metric &  $\lambda$ &         \texttt{DCBM-LS} &       \texttt{DCBM-NC-LS} &       \texttt{DCBM-NT-LS} &        \texttt{DCBM-NLS} &      \texttt{DCBM-NC-NLS} &      \texttt{DCBM-NT-NLS} \\
\midrule
\midrule\multirow{12}{*}{\rotatebox[origin=c]{90}{$AccTask$}} &       0.00 & $\mathbf{1.000\pm.000}$ & $\mathbf{1.000\pm.000}$ &           $.925\pm.010$ & $\mathbf{1.000\pm.000}$ &           $.997\pm.007$ &           $.909\pm.030$ \\
                                                              &       0.01 & $\mathbf{1.000\pm.000}$ &           $.998\pm.004$ &           $.918\pm.021$ & $\mathbf{1.000\pm.000}$ &           $.987\pm.013$ &           $.912\pm.019$ \\
                                                              &       0.05 &  $\mathbf{.999\pm.002}$ &           $.975\pm.013$ &           $.902\pm.025$ &           $.998\pm.004$ &           $.969\pm.022$ &           $.906\pm.025$ \\
                                                              &       0.10 &  $\mathbf{.986\pm.004}$ &           $.913\pm.021$ &           $.900\pm.031$ &           $.971\pm.018$ &           $.904\pm.013$ &           $.912\pm.020$ \\
                                                              &       0.15 &  $\mathbf{.966\pm.015}$ &           $.881\pm.011$ &           $.902\pm.012$ &           $.965\pm.011$ &           $.892\pm.015$ &           $.902\pm.019$ \\
                                                              &       0.20 &  $\mathbf{.925\pm.022}$ &           $.879\pm.010$ &           $.896\pm.013$ &           $.922\pm.012$ &           $.878\pm.010$ &           $.891\pm.012$ \\
                                                              &       0.25 &  $\mathbf{.915\pm.014}$ &           $.869\pm.020$ &           $.884\pm.016$ &           $.911\pm.010$ &           $.865\pm.009$ &           $.870\pm.022$ \\
\bottomrule
\end{tabular}}
\end{table}
\begin{table}[t]
\centering
\caption{Results for the \texttt{completeness} dataset when not allowing for shared parameters with joint training using OVA, and considering oracle task expert and oracle concept expert. We report $avg\pm std$ and highlight the best baseline in bold.}
\label{tab:frozen-CEoracle-TEoracle-joint-OVA-results}
\resizebox{\textwidth}{!}{\begin{tabular}{c|c|cccccc}
\toprule
                                                       Metric &  $\lambda$ &         \texttt{DCBM-LS} &       \texttt{DCBM-NC-LS} &       \texttt{DCBM-NT-LS} &        \texttt{DCBM-NLS} &      \texttt{DCBM-NC-NLS} &      \texttt{DCBM-NT-NLS} \\
\midrule
\midrule\multirow{12}{*}{\rotatebox[origin=c]{90}{$AccTask$}} &       0.00 & $\mathbf{1.000\pm.000}$ &           $.998\pm.004$ &           $.924\pm.022$ & $\mathbf{1.000\pm.000}$ & $\mathbf{1.000\pm.000}$ &           $.919\pm.017$ \\
                                                              &       0.01 & $\mathbf{1.000\pm.000}$ &           $.999\pm.002$ &           $.912\pm.009$ & $\mathbf{1.000\pm.000}$ & $\mathbf{1.000\pm.000}$ &           $.932\pm.010$ \\
                                                              &       0.05 &  $\mathbf{.996\pm.005}$ &           $.960\pm.021$ &           $.911\pm.024$ &  $\mathbf{.996\pm.007}$ &           $.969\pm.021$ &           $.895\pm.019$ \\
                                                              &       0.10 &           $.969\pm.027$ &           $.913\pm.015$ &           $.921\pm.005$ &  $\mathbf{.979\pm.004}$ &           $.914\pm.019$ &           $.912\pm.024$ \\
                                                              &       0.15 &           $.954\pm.019$ &           $.879\pm.011$ &           $.912\pm.015$ &  $\mathbf{.967\pm.014}$ &           $.886\pm.028$ &           $.909\pm.013$ \\
                                                              &       0.20 &  $\mathbf{.948\pm.013}$ &           $.885\pm.012$ &           $.890\pm.024$ &  $\mathbf{.948\pm.010}$ &           $.885\pm.019$ &           $.890\pm.007$ \\
                                                              &       0.25 &  $\mathbf{.917\pm.013}$ &           $.858\pm.006$ &           $.877\pm.012$ &           $.915\pm.019$ &           $.867\pm.013$ &           $.876\pm.017$ \\
\bottomrule
\end{tabular}}
\end{table}

\begin{table}
\centering
\caption{Results for the \texttt{completeness} dataset when allowing for shared parameters with independent training using ASM, and considering oracle task expert and oracle concept expert. LS refers to the label-smoothing-free implementation, while NLS to the one with label smoothing. We report $avg\pm std$ and highlight the best baseline in bold.}
\label{tab:unfrozen-CEoracle-TEoracle-independent-ASM-results}
\resizebox{\textwidth}{!}{\begin{tabular}{c|c|cccccc}
\toprule
                                                       Metric &  $\lambda$ &       \texttt{DCBM-LS} &     \texttt{DCBM-NC-LS} &     \texttt{DCBM-NT-LS} &      \texttt{DCBM-NLS} &     \texttt{DCBM-NC-NLS} &     \texttt{DCBM-NT-NLS} \\
\midrule
\midrule\multirow{12}{*}{\rotatebox[origin=c]{90}{$AccTask$}} &       0.00 & $\mathbf{1.000\pm.000}$ & $\mathbf{1.000\pm.000}$ &           $.906\pm.027$ & $\mathbf{1.000\pm.000}$ & $\mathbf{1.000\pm.000}$ &           $.892\pm.026$ \\
                                                              &       0.01 & $\mathbf{1.000\pm.000}$ & $\mathbf{1.000\pm.000}$ &           $.934\pm.004$ & $\mathbf{1.000\pm.000}$ & $\mathbf{1.000\pm.000}$ &           $.918\pm.028$ \\
                                                              &       0.05 &           $.992\pm.006$ &           $.961\pm.028$ &           $.922\pm.006$ &  $\mathbf{.996\pm.004}$ &           $.970\pm.016$ &           $.925\pm.012$ \\
                                                              &       0.10 &           $.969\pm.022$ &           $.920\pm.023$ &           $.897\pm.020$ &  $\mathbf{.978\pm.004}$ &           $.921\pm.050$ &           $.924\pm.009$ \\
                                                              &       0.15 &           $.943\pm.008$ &           $.874\pm.013$ &           $.910\pm.022$ &  $\mathbf{.948\pm.010}$ &           $.885\pm.033$ &           $.909\pm.012$ \\
                                                              &       0.20 &           $.895\pm.022$ &           $.856\pm.011$ &           $.875\pm.015$ &  $\mathbf{.905\pm.024}$ &           $.844\pm.007$ &           $.867\pm.016$ \\
                                                              &       0.25 &  $\mathbf{.854\pm.012}$ &           $.843\pm.004$ &           $.842\pm.020$ &           $.849\pm.024$ &           $.844\pm.019$ &           $.849\pm.016$ \\
\bottomrule
\end{tabular}}
\end{table}
\begin{table}[t]
\centering
\caption{Results for the \texttt{completeness} dataset when allowing for shared parameters with joint training using ASM, and considering oracle task expert and oracle concept expert. We report $avg\pm std$ and highlight the best baseline in bold.}
\label{tab:unfrozen-CEoracle-TEoracle-joint-ASM-results}
\resizebox{\textwidth}{!}{\begin{tabular}{c|c|cccccc}
\toprule
                                                       Metric &  $\lambda$ &       \texttt{DCBM-LS} &     \texttt{DCBM-NC-LS} &     \texttt{DCBM-NT-LS} &      \texttt{DCBM-NLS} &     \texttt{DCBM-NC-NLS} &     \texttt{DCBM-NT-NLS} \\
\midrule
\midrule\multirow{12}{*}{\rotatebox[origin=c]{90}{$AccTask$}} &       0.00 & $\mathbf{1.000\pm.000}$ & $\mathbf{1.000\pm.000}$ &           $.921\pm.027$ & $\mathbf{1.000\pm.000}$ & $\mathbf{1.000\pm.000}$ &           $.925\pm.014$ \\
                                                              &       0.01 & $\mathbf{1.000\pm.000}$ & $\mathbf{1.000\pm.000}$ &           $.929\pm.016$ & $\mathbf{1.000\pm.000}$ & $\mathbf{1.000\pm.000}$ &           $.916\pm.023$ \\
                                                              &       0.05 &           $.995\pm.005$ &           $.978\pm.017$ &           $.905\pm.024$ & $\mathbf{1.000\pm.000}$ &           $.962\pm.021$ &           $.920\pm.014$ \\
                                                              &       0.10 &  $\mathbf{.981\pm.004}$ &           $.901\pm.022$ &           $.925\pm.014$ &           $.977\pm.006$ &           $.923\pm.046$ &           $.900\pm.023$ \\
                                                              &       0.15 &           $.948\pm.020$ &           $.873\pm.014$ &           $.888\pm.027$ &  $\mathbf{.956\pm.009}$ &           $.877\pm.035$ &           $.906\pm.012$ \\
                                                              &       0.20 &           $.893\pm.029$ &           $.843\pm.009$ &           $.874\pm.023$ &  $\mathbf{.894\pm.019}$ &           $.859\pm.007$ &           $.868\pm.026$ \\
                                                              &       0.25 &           $.843\pm.020$ &           $.842\pm.016$ &  $\mathbf{.852\pm.010}$ &           $.847\pm.020$ &           $.848\pm.013$ &           $.843\pm.010$ \\
\bottomrule
\end{tabular}}
\end{table}

\begin{table}
\centering
\caption{Results for the \texttt{completeness} dataset when not allowing for shared parameters with independent training using ASM, and considering oracle task expert and human80 concept expert. LS refers to the label-smoothing-free implementation, while NLS to the one with label smoothing. We report $avg\pm std$ and highlight the best baseline in bold.}
\label{tab:frozen-CEhuman80-TEoracle-independent-ASM-results}
\resizebox{\textwidth}{!}{\begin{tabular}{c|c|cccccc}
\toprule
                                                       Metric &  $\lambda$ &         \texttt{DCBM-LS} &       \texttt{DCBM-NC-LS} &       \texttt{DCBM-NT-LS} &        \texttt{DCBM-NLS} &      \texttt{DCBM-NC-NLS} &      \texttt{DCBM-NT-NLS} \\
\midrule
\midrule\multirow{12}{*}{\rotatebox[origin=c]{90}{$AccTask$}} &       0.00 & $\mathbf{1.000\pm.000}$ & $\mathbf{1.000\pm.000}$ &           $.816\pm.015$ & $\mathbf{1.000\pm.000}$ & $\mathbf{1.000\pm.000}$ &           $.821\pm.007$ \\
                                                              &       0.01 &           $.999\pm.002$ &           $.998\pm.004$ &           $.828\pm.010$ & $\mathbf{1.000\pm.000}$ & $\mathbf{1.000\pm.000}$ &           $.815\pm.020$ \\
                                                              &       0.05 &           $.976\pm.014$ &  $\mathbf{.987\pm.012}$ &           $.820\pm.017$ &           $.965\pm.020$ &           $.982\pm.012$ &           $.832\pm.026$ \\
                                                              &       0.10 &           $.906\pm.023$ &  $\mathbf{.922\pm.036}$ &           $.832\pm.019$ &           $.901\pm.030$ &           $.880\pm.022$ &           $.834\pm.005$ \\
                                                              &       0.15 &           $.883\pm.022$ &           $.880\pm.009$ &           $.820\pm.014$ &           $.863\pm.014$ &  $\mathbf{.884\pm.015}$ &           $.820\pm.015$ \\
                                                              &       0.20 &           $.839\pm.013$ &  $\mathbf{.861\pm.008}$ &           $.835\pm.015$ &           $.844\pm.019$ &           $.844\pm.020$ &           $.833\pm.015$ \\
                                                              &       0.25 &           $.824\pm.014$ &           $.824\pm.014$ &           $.823\pm.014$ &           $.824\pm.015$ &  $\mathbf{.838\pm.008}$ &           $.826\pm.013$ \\
\bottomrule
\end{tabular}}
\end{table}
\begin{table}[t]
\centering
\caption{Results for the \texttt{completeness} dataset when not allowing for shared parameters with joint training using ASM, and considering oracle task expert and human80 concept expert. We report $avg\pm std$ and highlight the best baseline in bold.}
\label{tab:frozen-CEhuman80-TEoracle-joint-ASM-results}
\resizebox{\textwidth}{!}{\begin{tabular}{c|c|cccccc}
\toprule
                                                       Metric &  $\lambda$ &         \texttt{DCBM-LS} &       \texttt{DCBM-NC-LS} &       \texttt{DCBM-NT-LS} &        \texttt{DCBM-NLS} &      \texttt{DCBM-NC-NLS} &      \texttt{DCBM-NT-NLS} \\
\midrule
\midrule\multirow{12}{*}{\rotatebox[origin=c]{90}{$AccTask$}} &       0.00 & $\mathbf{1.000\pm.000}$ & $\mathbf{1.000\pm.000}$ &           $.809\pm.007$ & $\mathbf{1.000\pm.000}$ & $\mathbf{1.000\pm.000}$ &           $.817\pm.016$ \\
                                                              &       0.01 &           $.997\pm.004$ & $\mathbf{1.000\pm.000}$ &           $.830\pm.007$ & $\mathbf{1.000\pm.000}$ & $\mathbf{1.000\pm.000}$ &           $.803\pm.012$ \\
                                                              &       0.05 &  $\mathbf{.983\pm.014}$ &           $.981\pm.021$ &           $.831\pm.014$ &  $\mathbf{.983\pm.012}$ &           $.980\pm.009$ &           $.821\pm.011$ \\
                                                              &       0.10 &           $.886\pm.018$ &           $.883\pm.017$ &           $.824\pm.033$ &  $\mathbf{.913\pm.030}$ &           $.909\pm.018$ &           $.819\pm.011$ \\
                                                              &       0.15 &  $\mathbf{.875\pm.011}$ &           $.867\pm.010$ &           $.834\pm.007$ &           $.870\pm.030$ &           $.865\pm.018$ &           $.837\pm.006$ \\
                                                              &       0.20 &           $.847\pm.019$ &  $\mathbf{.856\pm.016}$ &           $.819\pm.024$ &           $.841\pm.023$ &           $.845\pm.010$ &           $.827\pm.012$ \\
                                                              &       0.25 &  $\mathbf{.838\pm.025}$ &           $.829\pm.010$ &           $.825\pm.012$ &           $.832\pm.008$ &  $\mathbf{.838\pm.014}$ &           $.821\pm.015$ \\
\bottomrule
\end{tabular}}
\end{table}

\begin{table}
\centering
\caption{Results for the \texttt{completeness} dataset when not allowing for shared parameters with independent training using ASM, and considering oracle task expert and human60 concept expert. LS refers to the label-smoothing-free implementation, while NLS to the one with label smoothing. We report $avg\pm std$ and highlight the best baseline in bold.}
\label{tab:frozen-CEhuman60-TEoracle-independent-ASM-results}
\resizebox{\textwidth}{!}{\begin{tabular}{c|c|cccccc}
\toprule
                                                       Metric &  $\lambda$ &         \texttt{DCBM-LS} &       \texttt{DCBM-NC-LS} &       \texttt{DCBM-NT-LS} &        \texttt{DCBM-NLS} &      \texttt{DCBM-NC-NLS} &      \texttt{DCBM-NT-NLS} \\
\midrule
\midrule\multirow{12}{*}{\rotatebox[origin=c]{90}{$AccTask$}} &       0.00 & $\mathbf{1.000\pm.000}$ & $\mathbf{1.000\pm.000}$ &           $.823\pm.016$ & $\mathbf{1.000\pm.000}$ & $\mathbf{1.000\pm.000}$ &           $.829\pm.023$ \\
                                                              &       0.01 &           $.999\pm.002$ &           $.998\pm.004$ &           $.817\pm.019$ & $\mathbf{1.000\pm.000}$ & $\mathbf{1.000\pm.000}$ &           $.831\pm.012$ \\
                                                              &       0.05 &           $.983\pm.008$ &  $\mathbf{.987\pm.012}$ &           $.816\pm.005$ &           $.955\pm.025$ &           $.982\pm.012$ &           $.832\pm.013$ \\
                                                              &       0.10 &           $.909\pm.027$ &  $\mathbf{.922\pm.036}$ &           $.829\pm.029$ &           $.894\pm.029$ &           $.880\pm.022$ &           $.830\pm.006$ \\
                                                              &       0.15 &           $.881\pm.014$ &           $.880\pm.009$ &           $.826\pm.017$ &           $.870\pm.012$ &  $\mathbf{.884\pm.015}$ &           $.817\pm.013$ \\
                                                              &       0.20 &           $.844\pm.010$ &  $\mathbf{.861\pm.008}$ &           $.836\pm.016$ &           $.839\pm.014$ &           $.844\pm.020$ &           $.834\pm.010$ \\
                                                              &       0.25 &           $.827\pm.007$ &           $.824\pm.014$ &           $.816\pm.013$ &           $.823\pm.014$ &  $\mathbf{.838\pm.008}$ &           $.818\pm.020$ \\
\bottomrule
\end{tabular}}
\end{table}
\begin{table}[t]
\centering
\caption{Results for the \texttt{completeness} dataset when not allowing for shared parameters with joint training using ASM, and considering oracle task expert and human60 concept expert. We report $avg\pm std$ and highlight the best baseline in bold.}
\label{tab:frozen-CEhuman60-TEoracle-joint-ASM-results}
\resizebox{\textwidth}{!}{\begin{tabular}{c|c|cccccc}
\toprule
                                                       Metric &  $\lambda$ &         \texttt{DCBM-LS} &       \texttt{DCBM-NC-LS} &       \texttt{DCBM-NT-LS} &        \texttt{DCBM-NLS} &      \texttt{DCBM-NC-NLS} &      \texttt{DCBM-NT-NLS} \\
\midrule
\midrule\multirow{12}{*}{\rotatebox[origin=c]{90}{$AccTask$}} &       0.00 & $\mathbf{1.000\pm.000}$ & $\mathbf{1.000\pm.000}$ &           $.823\pm.020$ & $\mathbf{1.000\pm.000}$ & $\mathbf{1.000\pm.000}$ &           $.824\pm.022$ \\
                                                              &       0.01 &           $.997\pm.007$ & $\mathbf{1.000\pm.000}$ &           $.837\pm.015$ & $\mathbf{1.000\pm.000}$ & $\mathbf{1.000\pm.000}$ &           $.821\pm.014$ \\
                                                              &       0.05 &           $.981\pm.016$ &           $.981\pm.021$ &           $.823\pm.014$ &  $\mathbf{.983\pm.013}$ &           $.980\pm.009$ &           $.828\pm.018$ \\
                                                              &       0.10 &           $.881\pm.020$ &           $.883\pm.017$ &           $.837\pm.019$ &           $.905\pm.022$ &  $\mathbf{.909\pm.018}$ &           $.824\pm.011$ \\
                                                              &       0.15 &           $.866\pm.014$ &           $.867\pm.010$ &           $.825\pm.004$ &  $\mathbf{.871\pm.016}$ &           $.865\pm.018$ &           $.827\pm.023$ \\
                                                              &       0.20 &           $.848\pm.017$ &  $\mathbf{.856\pm.016}$ &           $.818\pm.013$ &           $.843\pm.032$ &           $.845\pm.010$ &           $.827\pm.012$ \\
                                                              &       0.25 &           $.836\pm.020$ &           $.829\pm.010$ &           $.830\pm.022$ &           $.836\pm.014$ &  $\mathbf{.838\pm.014}$ &           $.821\pm.014$ \\
\bottomrule
\end{tabular}}
\end{table}

\begin{table}
\centering
\caption{Results for the \texttt{completeness} dataset when not allowing for shared parameters with independent training using ASM, and considering human80 task expert and oracle concept expert. LS refers to the label-smoothing-free implementation, while NLS to the one with label smoothing. We report $avg\pm std$ and highlight the best baseline in bold.}
\label{tab:frozen-CEoracle-TEhuman80-independent-ASM-results}
\resizebox{\textwidth}{!}{\begin{tabular}{c|c|cccccc}
\toprule
                                                       Metric &  $\lambda$ &         \texttt{DCBM-LS} &       \texttt{DCBM-NC-LS} &       \texttt{DCBM-NT-LS} &        \texttt{DCBM-NLS} &      \texttt{DCBM-NC-NLS} &      \texttt{DCBM-NT-NLS} \\
\midrule
\midrule\multirow{12}{*}{\rotatebox[origin=c]{90}{$AccTask$}} &       0.00 &           $.918\pm.024$ &           $.842\pm.012$ &  $\mathbf{.925\pm.010}$ &           $.921\pm.022$ &           $.841\pm.021$ &           $.909\pm.030$ \\
                                                              &       0.01 &           $.905\pm.019$ &           $.841\pm.004$ &  $\mathbf{.919\pm.022}$ &           $.911\pm.016$ &           $.853\pm.008$ &           $.912\pm.019$ \\
                                                              &       0.05 &           $.898\pm.034$ &           $.829\pm.015$ &           $.906\pm.025$ &           $.905\pm.019$ &           $.841\pm.015$ &  $\mathbf{.907\pm.026}$ \\
                                                              &       0.10 &           $.890\pm.033$ &           $.837\pm.014$ &           $.904\pm.036$ &           $.885\pm.034$ &           $.825\pm.018$ &  $\mathbf{.919\pm.019}$ \\
                                                              &       0.15 &           $.877\pm.030$ &           $.830\pm.019$ &  $\mathbf{.909\pm.007}$ &           $.864\pm.015$ &           $.825\pm.025$ &           $.907\pm.016$ \\
                                                              &       0.20 &           $.847\pm.010$ &           $.834\pm.010$ &  $\mathbf{.892\pm.009}$ &           $.855\pm.013$ &           $.830\pm.012$ &           $.885\pm.014$ \\
                                                              &       0.25 &           $.839\pm.016$ &           $.819\pm.018$ &  $\mathbf{.863\pm.022}$ &           $.855\pm.023$ &           $.829\pm.008$ &           $.850\pm.026$ \\
\bottomrule
\end{tabular}}
\end{table}
\begin{table}[t]
\centering
\caption{Results for the \texttt{completeness} dataset when not allowing for shared parameters with joint training using ASM, and considering human80 task expert and oracle concept expert. We report $avg\pm std$ and highlight the best baseline in bold.}
\label{tab:frozen-CEoracle-TEhuman80-joint-ASM-results}
\resizebox{\textwidth}{!}{\begin{tabular}{c|c|cccccc}
\toprule
                                                       Metric &  $\lambda$ &         \texttt{DCBM-LS} &       \texttt{DCBM-NC-LS} &       \texttt{DCBM-NT-LS} &        \texttt{DCBM-NLS} &      \texttt{DCBM-NC-NLS} &      \texttt{DCBM-NT-NLS} \\
\midrule
\midrule\multirow{12}{*}{\rotatebox[origin=c]{90}{$AccTask$}} &       0.00 &  $\mathbf{.938\pm.012}$ &           $.846\pm.017$ &           $.924\pm.022$ &           $.917\pm.022$ &           $.842\pm.007$ &           $.919\pm.017$ \\
                                                              &       0.01 &           $.899\pm.026$ &           $.839\pm.017$ &           $.912\pm.009$ &           $.919\pm.013$ &           $.842\pm.020$ &  $\mathbf{.932\pm.010}$ \\
                                                              &       0.05 &           $.880\pm.022$ &           $.838\pm.012$ &  $\mathbf{.913\pm.025}$ &           $.887\pm.023$ &           $.835\pm.022$ &           $.898\pm.018$ \\
                                                              &       0.10 &           $.874\pm.013$ &           $.825\pm.010$ &  $\mathbf{.924\pm.007}$ &           $.895\pm.025$ &           $.830\pm.013$ &           $.917\pm.024$ \\
                                                              &       0.15 &           $.866\pm.009$ &           $.823\pm.016$ &  $\mathbf{.921\pm.017}$ &           $.888\pm.031$ &           $.825\pm.019$ &           $.910\pm.012$ \\
                                                              &       0.20 &           $.878\pm.025$ &           $.830\pm.019$ &           $.881\pm.020$ &           $.856\pm.013$ &           $.834\pm.012$ &  $\mathbf{.887\pm.014}$ \\
                                                              &       0.25 &           $.843\pm.018$ &           $.823\pm.012$ &           $.858\pm.018$ &           $.837\pm.016$ &           $.840\pm.005$ &  $\mathbf{.860\pm.008}$ \\
\bottomrule
\end{tabular}}
\end{table}

\begin{table}
\centering
\caption{Results for the \texttt{completeness} dataset when not allowing for shared parameters with independent training using ASM, and considering human80 task expert and human80 concept expert. LS refers to the label-smoothing-free implementation, while NLS to the one with label smoothing. We report $avg\pm std$ and highlight the best baseline in bold.}
\label{tab:frozen-CEhuman80-TEhuman80-independent-ASM-results}
\resizebox{\textwidth}{!}{\begin{tabular}{c|c|cccccc}
\toprule
                                                       Metric &  $\lambda$ &         \texttt{DCBM-LS} &       \texttt{DCBM-NC-LS} &       \texttt{DCBM-NT-LS} &        \texttt{DCBM-NLS} &      \texttt{DCBM-NC-NLS} &      \texttt{DCBM-NT-NLS} \\
\midrule
\midrule\multirow{12}{*}{\rotatebox[origin=c]{90}{$AccTask$}} &       0.00 &  $\mathbf{.861\pm.013}$ &           $.842\pm.012$ &           $.816\pm.015$ &           $.844\pm.020$ &           $.841\pm.021$ &           $.821\pm.007$ \\
                                                              &       0.01 &           $.834\pm.016$ &           $.841\pm.004$ &           $.828\pm.010$ &           $.833\pm.020$ &  $\mathbf{.853\pm.008}$ &           $.815\pm.020$ \\
                                                              &       0.05 &           $.838\pm.027$ &           $.829\pm.015$ &           $.820\pm.017$ &           $.827\pm.014$ &  $\mathbf{.841\pm.015}$ &           $.832\pm.026$ \\
                                                              &       0.10 &           $.824\pm.010$ &  $\mathbf{.837\pm.014}$ &           $.832\pm.019$ &           $.820\pm.014$ &           $.825\pm.018$ &           $.834\pm.005$ \\
                                                              &       0.15 &  $\mathbf{.837\pm.010}$ &           $.830\pm.019$ &           $.820\pm.014$ &           $.820\pm.015$ &           $.825\pm.025$ &           $.820\pm.015$ \\
                                                              &       0.20 &           $.817\pm.015$ &           $.834\pm.010$ &  $\mathbf{.835\pm.015}$ &           $.826\pm.013$ &           $.830\pm.012$ &           $.833\pm.015$ \\
                                                              &       0.25 &           $.821\pm.019$ &           $.819\pm.018$ &           $.823\pm.014$ &           $.827\pm.018$ &  $\mathbf{.829\pm.008}$ &           $.826\pm.013$ \\
\bottomrule
\end{tabular}}
\end{table}
\begin{table}[t]
\centering
\caption{Results for the \texttt{completeness} dataset when not allowing for shared parameters with joint training using ASM, and considering human80 task expert and human80 concept expert. We report $avg\pm std$ and highlight the best baseline in bold.}
\label{tab:frozen-CEhuman80-TEhuman80-joint-ASM-results}
\resizebox{\textwidth}{!}{\begin{tabular}{c|c|cccccc}
\toprule
                                                       Metric &  $\lambda$ &         \texttt{DCBM-LS} &       \texttt{DCBM-NC-LS} &       \texttt{DCBM-NT-LS} &        \texttt{DCBM-NLS} &      \texttt{DCBM-NC-NLS} &      \texttt{DCBM-NT-NLS} \\
\midrule
\midrule\multirow{12}{*}{\rotatebox[origin=c]{90}{$AccTask$}} &       0.00 &           $.837\pm.019$ &           $.846\pm.017$ &           $.809\pm.007$ &  $\mathbf{.857\pm.025}$ &           $.842\pm.007$ &           $.817\pm.016$ \\
                                                              &       0.01 &           $.835\pm.018$ &           $.839\pm.017$ &           $.830\pm.007$ &  $\mathbf{.844\pm.017}$ &           $.842\pm.020$ &           $.803\pm.012$ \\
                                                              &       0.05 &           $.836\pm.020$ &  $\mathbf{.838\pm.012}$ &           $.831\pm.014$ &           $.833\pm.008$ &           $.835\pm.022$ &           $.821\pm.011$ \\
                                                              &       0.10 &           $.824\pm.008$ &           $.825\pm.010$ &           $.824\pm.033$ &           $.825\pm.025$ &  $\mathbf{.830\pm.013}$ &           $.819\pm.011$ \\
                                                              &       0.15 &           $.814\pm.022$ &           $.823\pm.016$ &           $.834\pm.007$ &           $.823\pm.017$ &           $.825\pm.019$ &  $\mathbf{.837\pm.006}$ \\
                                                              &       0.20 &           $.827\pm.013$ &           $.830\pm.019$ &           $.819\pm.024$ &           $.822\pm.016$ &  $\mathbf{.834\pm.012}$ &           $.827\pm.012$ \\
                                                              &       0.25 &           $.830\pm.010$ &           $.823\pm.012$ &           $.825\pm.012$ &           $.817\pm.015$ &  $\mathbf{.840\pm.005}$ &           $.821\pm.015$ \\
\bottomrule
\end{tabular}}
\end{table}

\begin{table}
\centering
\caption{Results for the \texttt{completeness} dataset when not allowing for shared parameters with independent training using ASM, and considering human80 task expert and human60 concept expert. LS refers to the label-smoothing-free implementation, while NLS to the one with label smoothing. We report $avg\pm std$ and highlight the best baseline in bold.}
\label{tab:frozen-CEhuman60-TEhuman80-independent-ASM-results}
\resizebox{\textwidth}{!}{\begin{tabular}{c|c|cccccc}
\toprule
                                                       Metric &  $\lambda$ &         \texttt{DCBM-LS} &       \texttt{DCBM-NC-LS} &       \texttt{DCBM-NT-LS} &        \texttt{DCBM-NLS} &      \texttt{DCBM-NC-NLS} &      \texttt{DCBM-NT-NLS} \\
\midrule
\midrule\multirow{12}{*}{\rotatebox[origin=c]{90}{$AccTask$}} &       0.00 &           $.840\pm.008$ &  $\mathbf{.842\pm.012}$ &           $.823\pm.016$ &           $.839\pm.023$ &           $.841\pm.021$ &           $.829\pm.023$ \\
                                                              &       0.01 &           $.829\pm.013$ &           $.841\pm.004$ &           $.817\pm.019$ &           $.829\pm.010$ &  $\mathbf{.853\pm.008}$ &           $.831\pm.012$ \\
                                                              &       0.05 &           $.836\pm.020$ &           $.829\pm.015$ &           $.816\pm.005$ &           $.828\pm.017$ &  $\mathbf{.841\pm.015}$ &           $.832\pm.013$ \\
                                                              &       0.10 &           $.828\pm.016$ &  $\mathbf{.837\pm.014}$ &           $.829\pm.029$ &           $.828\pm.011$ &           $.825\pm.018$ &           $.830\pm.006$ \\
                                                              &       0.15 &  $\mathbf{.831\pm.009}$ &           $.830\pm.019$ &           $.826\pm.017$ &           $.827\pm.010$ &           $.825\pm.025$ &           $.817\pm.013$ \\
                                                              &       0.20 &           $.819\pm.016$ &           $.834\pm.010$ &  $\mathbf{.836\pm.016}$ &           $.828\pm.013$ &           $.830\pm.012$ &           $.834\pm.010$ \\
                                                              &       0.25 &           $.818\pm.018$ &           $.819\pm.018$ &           $.816\pm.013$ &           $.822\pm.010$ &  $\mathbf{.829\pm.008}$ &           $.818\pm.020$ \\
\bottomrule
\end{tabular}}
\end{table}
\begin{table}[t]
\centering
\caption{Results for the \texttt{completeness} dataset when not allowing for shared parameters with joint training using ASM, and considering human80 task expert and human60 concept expert. We report $avg\pm std$ and highlight the best baseline in bold.}
\label{tab:frozen-CEhuman60-TEhuman80-joint-ASM-results}
\resizebox{\textwidth}{!}{\begin{tabular}{c|c|cccccc}
\toprule
                                                       Metric &  $\lambda$ &         \texttt{DCBM-LS} &       \texttt{DCBM-NC-LS} &       \texttt{DCBM-NT-LS} &        \texttt{DCBM-NLS} &      \texttt{DCBM-NC-NLS} &      \texttt{DCBM-NT-NLS} \\
\midrule
\midrule\multirow{12}{*}{\rotatebox[origin=c]{90}{$AccTask$}} &       0.00 &  $\mathbf{.846\pm.023}$ &  $\mathbf{.846\pm.017}$ &           $.823\pm.020$ &           $.834\pm.013$ &           $.842\pm.007$ &           $.824\pm.022$ \\
                                                              &       0.01 &           $.839\pm.016$ &           $.839\pm.017$ &           $.837\pm.015$ &           $.834\pm.013$ &  $\mathbf{.842\pm.020}$ &           $.821\pm.014$ \\
                                                              &       0.05 &           $.830\pm.015$ &  $\mathbf{.838\pm.012}$ &           $.823\pm.014$ &           $.827\pm.009$ &           $.835\pm.022$ &           $.828\pm.018$ \\
                                                              &       0.10 &           $.826\pm.016$ &           $.825\pm.010$ &  $\mathbf{.837\pm.019}$ &           $.822\pm.010$ &           $.830\pm.013$ &           $.824\pm.011$ \\
                                                              &       0.15 &           $.826\pm.026$ &           $.823\pm.016$ &           $.825\pm.004$ &           $.821\pm.012$ &           $.825\pm.019$ &  $\mathbf{.827\pm.023}$ \\
                                                              &       0.20 &  $\mathbf{.835\pm.017}$ &           $.830\pm.019$ &           $.818\pm.013$ &           $.818\pm.016$ &           $.834\pm.012$ &           $.827\pm.012$ \\
                                                              &       0.25 &           $.823\pm.018$ &           $.823\pm.012$ &           $.830\pm.022$ &           $.819\pm.008$ &  $\mathbf{.840\pm.005}$ &           $.821\pm.014$ \\
\bottomrule
\end{tabular}}
\end{table}

\begin{table}
\centering
\caption{Results for the \texttt{completeness} dataset when not allowing for shared parameters with independent training using ASM, and considering human60 task expert and oracle concept expert. LS refers to the label-smoothing-free implementation, while NLS to the one with label smoothing. We report $avg\pm std$ and highlight the best baseline in bold.}
\label{tab:frozen-CEoracle-TEhuman60-independent-ASM-results}
\resizebox{\textwidth}{!}{\begin{tabular}{c|c|cccccc}
\toprule
                                                       Metric &  $\lambda$ &         \texttt{DCBM-LS} &       \texttt{DCBM-NC-LS} &       \texttt{DCBM-NT-LS} &        \texttt{DCBM-NLS} &      \texttt{DCBM-NC-NLS} &      \texttt{DCBM-NT-NLS} \\
\midrule
\midrule\multirow{12}{*}{\rotatebox[origin=c]{90}{$AccTask$}} &       0.00 &           $.884\pm.032$ &           $.826\pm.016$ &  $\mathbf{.925\pm.010}$ &           $.872\pm.004$ &           $.822\pm.013$ &           $.909\pm.030$ \\
                                                              &       0.01 &           $.890\pm.022$ &           $.835\pm.011$ &  $\mathbf{.919\pm.022}$ &           $.918\pm.023$ &           $.834\pm.011$ &           $.912\pm.019$ \\
                                                              &       0.05 &           $.886\pm.029$ &           $.831\pm.016$ &           $.906\pm.025$ &           $.898\pm.023$ &           $.839\pm.007$ &  $\mathbf{.907\pm.026}$ \\
                                                              &       0.10 &           $.892\pm.030$ &           $.839\pm.017$ &           $.904\pm.036$ &           $.887\pm.030$ &           $.824\pm.013$ &  $\mathbf{.919\pm.019}$ \\
                                                              &       0.15 &           $.867\pm.033$ &           $.828\pm.015$ &  $\mathbf{.909\pm.007}$ &           $.873\pm.037$ &           $.828\pm.018$ &           $.907\pm.016$ \\
                                                              &       0.20 &           $.848\pm.006$ &           $.829\pm.013$ &  $\mathbf{.892\pm.009}$ &           $.854\pm.035$ &           $.831\pm.014$ &           $.885\pm.014$ \\
                                                              &       0.25 &           $.836\pm.014$ &           $.812\pm.017$ &  $\mathbf{.863\pm.022}$ &           $.847\pm.019$ &           $.833\pm.006$ &           $.850\pm.026$ \\
\bottomrule
\end{tabular}}
\end{table}
\begin{table}[t]
\centering
\caption{Results for the \texttt{completeness} dataset when not allowing for shared parameters with joint training using ASM, and considering human60 task expert and oracle concept expert. We report $avg\pm std$ and highlight the best baseline in bold.}
\label{tab:frozen-CEoracle-TEhuman60-joint-ASM-results}
\resizebox{\textwidth}{!}{\begin{tabular}{c|c|cccccc}
\toprule
                                                       Metric &  $\lambda$ &         \texttt{DCBM-LS} &       \texttt{DCBM-NC-LS} &       \texttt{DCBM-NT-LS} &        \texttt{DCBM-NLS} &      \texttt{DCBM-NC-NLS} &      \texttt{DCBM-NT-NLS} \\
\midrule
\midrule\multirow{12}{*}{\rotatebox[origin=c]{90}{$AccTask$}} &       0.00 &           $.895\pm.015$ &           $.828\pm.006$ &  $\mathbf{.924\pm.022}$ &           $.885\pm.031$ &           $.809\pm.007$ &           $.919\pm.017$ \\
                                                              &       0.01 &           $.881\pm.026$ &           $.822\pm.016$ &           $.912\pm.009$ &           $.900\pm.023$ &           $.827\pm.018$ &  $\mathbf{.932\pm.010}$ \\
                                                              &       0.05 &           $.894\pm.022$ &           $.832\pm.014$ &  $\mathbf{.913\pm.025}$ &           $.873\pm.006$ &           $.831\pm.017$ &           $.898\pm.018$ \\
                                                              &       0.10 &           $.867\pm.013$ &           $.822\pm.010$ &  $\mathbf{.924\pm.007}$ &           $.905\pm.022$ &           $.826\pm.007$ &           $.917\pm.024$ \\
                                                              &       0.15 &           $.877\pm.025$ &           $.821\pm.014$ &  $\mathbf{.921\pm.017}$ &           $.873\pm.024$ &           $.818\pm.011$ &           $.910\pm.012$ \\
                                                              &       0.20 &           $.872\pm.027$ &           $.840\pm.012$ &           $.881\pm.020$ &           $.854\pm.018$ &           $.822\pm.008$ &  $\mathbf{.887\pm.014}$ \\
                                                              &       0.25 &           $.840\pm.022$ &           $.822\pm.014$ &           $.858\pm.018$ &           $.835\pm.015$ &           $.831\pm.010$ &  $\mathbf{.860\pm.008}$ \\
\bottomrule
\end{tabular}}
\end{table}

\input{assets/tabs/frozen_independent_ASM_CEhuman80_TEhuman60}
\begin{table}[t]
\centering
\caption{Results for the \texttt{completeness} dataset when not allowing for shared parameters with joint training using ASM, and considering human60 task expert and human80 concept expert. We report $avg\pm std$ and highlight the best baseline in bold.}
\label{tab:frozen-CEhuman80-TEhuman60-joint-ASM-results}
\resizebox{\textwidth}{!}{\begin{tabular}{c|c|cccccc}
\toprule
                                                       Metric &  $\lambda$ &         \texttt{DCBM-LS} &       \texttt{DCBM-NC-LS} &       \texttt{DCBM-NT-LS} &        \texttt{DCBM-NLS} &      \texttt{DCBM-NC-NLS} &      \texttt{DCBM-NT-NLS} \\
\midrule
\midrule\multirow{12}{*}{\rotatebox[origin=c]{90}{$AccTask$}} &       0.00 &           $.809\pm.014$ &  $\mathbf{.828\pm.006}$ &           $.809\pm.007$ &           $.825\pm.018$ &           $.809\pm.007$ &           $.817\pm.016$ \\
                                                              &       0.01 &           $.818\pm.014$ &           $.822\pm.016$ &  $\mathbf{.830\pm.007}$ &           $.822\pm.012$ &           $.827\pm.018$ &           $.803\pm.012$ \\
                                                              &       0.05 &           $.830\pm.006$ &  $\mathbf{.832\pm.014}$ &           $.831\pm.014$ &           $.831\pm.013$ &           $.831\pm.017$ &           $.821\pm.011$ \\
                                                              &       0.10 &           $.821\pm.014$ &           $.822\pm.010$ &           $.824\pm.033$ &           $.823\pm.026$ &  $\mathbf{.826\pm.007}$ &           $.819\pm.011$ \\
                                                              &       0.15 &           $.820\pm.026$ &           $.821\pm.014$ &           $.834\pm.007$ &           $.822\pm.011$ &           $.818\pm.011$ &  $\mathbf{.837\pm.006}$ \\
                                                              &       0.20 &           $.829\pm.015$ &  $\mathbf{.840\pm.012}$ &           $.819\pm.024$ &           $.821\pm.013$ &           $.822\pm.008$ &           $.827\pm.012$ \\
                                                              &       0.25 &           $.826\pm.008$ &           $.822\pm.014$ &           $.825\pm.012$ &           $.819\pm.013$ &  $\mathbf{.831\pm.010}$ &           $.821\pm.015$ \\
\bottomrule
\end{tabular}}
\end{table}

\input{assets/tabs/frozen_independent_ASM_CEhuman60_TEhuman60}
\input{assets/tabs/frozen_joint_ASM_CEhuman60_TEhuman60}

\end{document}